\documentclass[11pt]{article}
\usepackage{graphicx}
\usepackage{fullpage}
\usepackage{amsmath,amssymb,amsthm}
\usepackage{mathrsfs}
\usepackage{verbatim}
\usepackage{bbm}
\usepackage{marginnote} 
\usepackage[T1]{fontenc}
\usepackage{kpfonts}
\usepackage{microtype}
\usepackage{enumitem}
\usepackage{authblk}
\usepackage{booktabs}
\usepackage{multirow}
\usepackage[usenames]{color}
\definecolor{plum}{rgb}{.4,0,.4}
\definecolor{BrickRed}{rgb}{0.6,0,0}
\usepackage[round]{natbib}
\usepackage[plainpages=false,pdfpagelabels,colorlinks=true,linkcolor=BrickRed,citecolor=plum]{hyperref}
\usepackage[ruled,vlined]{algorithm2e}

\numberwithin{equation}{section}
\newtheorem{theorem}{Theorem}
\newtheorem{corollary}[theorem]{Corollary}
\newtheorem{lemma}[theorem]{Lemma}
\newtheorem{assumption}{Assumption}
\newtheorem{proposition}[theorem]{Proposition}
\newtheorem{definition}[theorem]{Definition}

\newtheorem{remark}[theorem]{Remark}

\numberwithin{theorem}{section}

\def\ddefloop#1{\ifx\ddefloop#1\else\ddef{#1}\expandafter\ddefloop\fi}
\def\ddef#1{\expandafter\def\csname c#1\endcsname{\ensuremath{\mathcal{#1}}}}
\ddefloop ABCDEFGHIJKLMNOPQRSTUVWXYZ\ddefloop
\def\ddef#1{\expandafter\def\csname s#1\endcsname{\ensuremath{\mathsf{#1}}}}
\ddefloop ABCDEFGHIJKLMNOPQRSTUVWXYZ\ddefloop

\def\PP{\mathbf{P}}

\def\K{\mathbf{K}}

\def\g{\mathbf{g}}
\def\t{\pmb{\theta}}

\newcommand{\Var}{\mathrm{\bf Var}}
\newcommand{\norm}[1]{\left\| #1 \right\|}
\newcommand{\vol}{\mathrm{vol}}
\newcommand{\EE}{\mathbb E}

\def\I{\mathbf{I}}

\begin{document}

\title{Mehler's Formula, Branching Process, and Compositional Kernels of Deep Neural Networks}

\author{Tengyuan Liang\thanks{{\tt tengyuan.liang@chicagobooth.edu}}~}
\author{Hai Tran-Bach\thanks{{\tt tranbach@uchicago.edu}}~}
\affil{University of Chicago}

\maketitle
\thispagestyle{empty}
\maketitle

\begin{abstract}
	We utilize a connection between compositional kernels and branching processes via Mehler's formula to study deep neural networks. This new probabilistic insight provides us a novel perspective on the mathematical role of activation functions in compositional neural networks. We study the unscaled and rescaled limits of the compositional kernels and explore the different phases of the limiting behavior, as the compositional depth increases. We investigate the memorization capacity of the compositional kernels and neural networks by characterizing the interplay among compositional depth, sample size, dimensionality, and non-linearity of the activation. Explicit formulas on the eigenvalues of the compositional kernel are provided, which quantify the complexity of the corresponding reproducing kernel Hilbert space. On the methodological front, we propose a new random features algorithm, which compresses the compositional layers by devising a new activation function.

\end{abstract}

\section{Introduction}

Kernel methods and deep neural networks are arguably two representative methods that achieved the state-of-the-art results in regression and classification tasks \citep{shankar2020neural}. However, unlike the kernel methods where both the statistical and computational aspects of learning have been understood reasonably well, there are still many theoretical puzzles around the generalization, computation and representation aspects of deep neural networks \citep{ZBH+17}. One hopeful direction to resolve some of the puzzles in neural networks is through the lens of kernels \citep{RR08, RR09,CS09,BMM18}. Such a connection can be readily observed in a two-layer infinite-width network with random weights, see the pioneering work by \cite{Nea96} and \cite{RR08, RR09}. For deep networks with hierarchical structures and randomly initialized weights, compositional kernels \citep{DFS17,DFS17,NIPS2016_6322} are proposed to rigorously characterize such a connection, with promising empirical performances \citep{CS09}. A list of simple algebraic operations on kernels \citep{stitson1997anova,shankar2020neural} are introduced to incorporate specific data structures that contain bag-of-features, such as images and time series.

In this paper, we continue to study deep neural networks and their dual compositional kernels, furthering the aforementioned mathematical connection, based on the foundational work of \citep{RR08, RR09} and \citep{DFS17,DFGS17}. We focus on a standard multilayer perceptron architecture with Gaussian weights and study the role of the activation function and its effect on composition, data memorization, spectral properties, algorithms, among others. Our main results are based on a simple yet elegant connection between \textit{compositional kernels} and \textit{branching processes} via Mehler's formula (Lemma \ref{lem:key-duality}). This new connection, in turn, opens up the possibility of studying the mathematical role of activation functions in compositional deep neural networks, utilizing the probabilistic tools in branching processes (Theorem \ref{thm:rescaled-limit}). Specifically, the new probabilistic insight allows us to answer the following questions:

\noindent \textbf{Limits and phase transitions.} Given an activation function, one can define the corresponding compositional kernel \citep{DFS17, DFGS17}. How to classify the activation functions according to the limits of their dual compositional kernels, as the compositional depth increases? What properties of the activation functions govern the different phases of such limits? How do we properly rescale the compositional kernel such that there is a limit unique to the activation function? The above questions will be explored in Section~\ref{sec:compo-kernel-branching-proc}.

\noindent \textbf{Memorization capacity of compositions: tradeoffs.} Deep neural networks and kernel machines can have a good out-of-sample performance even in the interpolation regime \citep{ZBH+17, BMM18}, with perfect memorization of the training dataset. What is the memorization capacity of the compositional kernels? What are the tradeoffs among compositional depth, number of samples in the dataset, input dimensionality, and properties of the non-linear activation functions? Section~\ref{sec:nonasymptotic-memorization} studies such interplay explicitly.

\noindent \textbf{Spectral properties of compositional kernels.} Spectral properties of the kernel (and the corresponding integral operator) affect the statistical rate of convergence, for kernel regressions \citep{CV06}. What is the spectral decomposition of the compositional kernels? How do the eigenvalues of the compositional kernel depend on the activation function? Section~\ref{sec:spherical-harmonics} is devoted to answering the above questions.

\noindent \textbf{New randomized algorithms.} Given a compositional kernel with a finite depth associate with an activation, can we devise a new "compressed" activation and new randomized algorithms, such that the deep neural network (with random weights) with the original activation is equivalent to a shallow neural network with the "compressed" activation? Such algorithmic questions are closely related to the seminal Random Fourier Features (RFF) algorithm in \cite{RR08, RR09}, yet different.
 Section~\ref{sec:new-random-feature-alg} investigates such algorithmic questions by considering compositional kernels and, more broadly, the inner-product kernels. Differences to the RFF are also discussed in detail therein.

Borrowing the insight from branching process, we start with studying the role of activation function in the compositional kernel, memorization capacity, and spectral properties, and conclude with the converse question of designing new activations and random nonlinear features algorithm based on kernels, thus contributing to a strengthened mathematical understanding of activation functions, compositional kernel classes, and deep neural networks.

\subsection{Related Work}

The connections between neural networks (with random weights) and kernel methods have been formalized by researchers using different mathematical languages. Instead of aiming to provide a complete list, here we only highlight a few that directly motivate our work. \cite{Nea96a, Nea96} advocated using Gaussian processes to characterize the neural networks with random weights from a Bayesian viewpoint. For two-layer neural networks, such correspondence has been strengthened mathematically by the work of \cite{RR08, RR09}. By Bochner's Theorem, \cite{RR08} showed that any positive definite translation-invariant kernel could be realized by a two-layer neural network with a specific distribution on the weights, via trigonometric activations. Such insights also motivated the well-known random features algorithm, random kitchen sinks \citep{RR09}. One highlight of such an algorithm is that in the first layer of weights, sampling is employed to replace the optimization. Later, several works extended along the line, see, for instance, \cite{KK12} on the rotation-invariant kernels, \cite{PYK15} on the polynomial kernels, and \cite{Bac16} on kernels associated to ReLU-like activations (using spherical harmonics). Recently, \cite{MM19} investigated the precise asymptotics of the random features model using random matrix theory. For deep neural networks, compositional kernels are proposed to carry such connections further. \cite{CS09} introduced the compositional kernel as the inner-product of compositional features. \cite{DFS17, DFGS17} described the compositional kernel through the language of the computational skeleton, and introduced the duality between the activation function and compositional kernel. We refer the readers to \cite{NIPS2016_6322, yang2019scaling,shankar2020neural} for more information on the connection between kernels and neural networks.

One might argue that neural networks with static random weights may not fully explain the success of neural networks, noticing that the evolution of the weights during training is yet another critical component.
On this front, \cite{CB18a, MMN18, SS18, RV18} employed the mean-field characterization to describe the distribution dynamics of the weights, for two-layer networks. \cite{RV18, DL19} studied the favorable properties of the dynamic kernel due to the evolution of the weight distribution. \cite{nguyen2020rigorous} carried the mean-field analysis to multi-layer networks rigorously. On a different tread \citep{JGH19, DZPS18, CB18, WGS+19}, researchers showed that under specific scaling, training over-parametrized networks could be viewed as a kernel regression with perfect memorization of the training data, using a tangent kernel \citep{JGH19} built from a linearization around its initialization. For a more recent resemblance between the kernel learning and the deep learning on the empirical side, we refer the readers to \cite{BMM18}.

\section{Preliminary}

\paragraph{Mehler's formula.}
We will start with reviewing some essential background on the Hermite polynomials that is of direct relevance to our paper.
\begin{definition}[Hermite polynomials]
  \label{def:hermite}
  The probabilists' Hermite polynomials $He_k(x)$ for non-negative integers $k \in  \mathbb{Z}^{\geq 0}$ follows the recursive definition with $He_0(x) = 1$ and
  \begin{align}
    He_{k+1} (x) = x He_{k}(x) - He_{k}'(x) \enspace.
  \end{align}
  We define the normalized Hermite polynomials as
  \begin{align}
    h_k(x) := \frac{1}{\sqrt{k!}} He_k(x), ~~\text{with}~\EE_{\g \sim \cN(0, 1)}[ h_k^2(\g)] = 1 \enspace.
  \end{align}
  The set $\{ h_k\colon k\in\mathbb{Z}^{\geq 0}\}$ forms an orthogonal basis of $L^2_{\phi}$ under the Gaussian measure $\phi \sim \cN(0, 1)$ as 
  \begin{align}
    \EE_{\g \sim \cN(0, 1)}[ h_k(\g) h_{k'}(\g) ] = \mathbbm{1}_{k = k'} \enspace.
  \end{align}
\end{definition}

\begin{proposition}[Mehler's formula]
  \label{prop:mehler}
  Mehler's formula establishes the following equality on Hermite polynomials: for any $\rho \in (-1,1)$ and $x, y \in \mathbb{R}$
  \begin{align}
    \frac{1}{\sqrt{1-\rho^2}} \exp\left(-\frac{\rho^2(x^2+y^2) - 2\rho xy}{2(1-\rho^2)} \right) = \sum_{k=0}^{\infty} \rho^k h_{k}(x) h_{k}(y) \enspace.
  \end{align}
\end{proposition}

\paragraph{Branching process.} Now, we will describe the branching process and the compositions of probability generating functions (PGF).
\begin{definition}[Probability generating function]
  \label{def:PGF}
  Given a random variable $Y$ on non-negative integers with the following probability distribution
  \begin{align}
    \mathbb{P}(Y = k) = p_k, ~~\forall k \in \mathbb{Z}^{\geq 0}\enspace,
  \end{align}
  define the associated generating function as
  \begin{align}
    G_{Y}(s) := \EE[s^{Y}] = \sum_{k\geq 0} p_k s^k \enspace.
  \end{align}
  It is clear that $G_{Y}(0) = 0$, $G_{Y}(1) = 1$ and $G_{Y}(s)$ is non-decreasing and convex on $s \in [0,1]$.
\end{definition}
\
\begin{definition}[Galton-Watson branching process]
  \label{def:galton-watson}
  The Galton-Watson (GW) branching process is defined as a Markov chain $\{ Z_L\colon L\in\mathbb{Z}^{\geq 0}\}$, where $Z_L$ denotes the size of the $L$-th generation of the initial family.
  Let $Y$ be a random variable on non-negative integers describing the number of direct children, that is, it has $k$ children with probability $p_k$ with $\sum_{k \geq 0} p_k = 1$. Begin with one individual $Z_0 \equiv 1$, and let it reproduce according to the distribution of $Y$, and then each of these children then reproduce independently with the same distribution as $Y$. The generation sizes $\{Z_{L}\colon L\in\mathbb{Z}^{\geq 0}\}$ are then defined by
  \begin{align}
    Z_{L+1} = \sum_{i=1}^{Z_{L}} Y^{(L)}_{i}
  \end{align}
  where $Y^{(L)}_{i}$ denotes the number of children for the $i$-th individual in generation $L$.
\end{definition}

\begin{proposition}[Extinction criterion, Proposition 5.4 in \cite{LP16}]
  \label{prop:extinction-criterion}
  Given $p_1 \neq 1$ (in Definition~\ref{def:PGF}), the extinction probability $\xi := \lim_{L \rightarrow \infty} \PP\left( Z_{L} = 0 \right)$ satisfies
  \begin{enumerate}
    \item[(i)] $\xi = 1$ if and only if $\mu:= G_{Y}'(1) \leq 1$.
    \item[(ii)] $\xi$ is the unique fixed point of $G_{Y}(s) = s$ in $[0,1)$.
  \end{enumerate}
\end{proposition}

\paragraph{Multi-layer Perceptrons.} We now define the fully-connected Multi-Layer Perceptrons (MLPs), which is among the standard architectures in deep neural networks.

\begin{definition}[Activation function]
  \label{asmp:activation}
  Throughout the paper, we will only consider the activation functions $\sigma(\cdot) \in L^2_{\phi} \colon\mathbb{R}\to\mathbb{R}$ that are $L^2$-integrable under the Gaussian measure $\phi$. The Hermite expansion of $\sigma(\cdot)$ is denoted as 
  \begin{align}
    \sigma(x) := \sum_{k\geq 0} a_k h_k(x)\enspace.
  \end{align}
\end{definition}
We will explicitly mention the following two assumptions when they are assumed. Otherwise, we will work with the activation $\sigma(\cdot)$ in Definition~$\ref{asmp:activation}$.

\begin{assumption}[Normalized activation function]
  \label{asmp:norm-activation}
  Assume that the activation function $\sigma(\cdot)\in L_{\phi}^2(1)$ is normalized under the Gaussian measure $\phi$, in the following sense
    \begin{align}
    \EE_{\g \sim \cN(0,1)}[ \sigma^2(\g) ] = 1 \enspace,
  \end{align}
  with the Hermite coefficients satisfying
  
    \begin{align}
    \sum_{k\geq 0} a_k^2=1.
    \end{align}
\end{assumption}

\begin{assumption} [Centered activation function]
  \label{asmp:cent-activation}
  Assume that the activation function $\sigma(\cdot)\in L_{\phi}^2$ is centered under the Gaussian measure $\phi$, in the following sense 
    \begin{align}
    \EE_{\g \sim \cN(0,1)}[ \sigma(\g) ] = 0, \quad 
  \end{align}
  or equivalently the Hermite coefficient $a_{0}=0$.
\end{assumption}

\begin{remark} Any activation $\sigma(\cdot)\in L^{2}_{\phi}$ that are $L^2$ integrable under the Gaussian measure $\phi$ can be re-centered and re-scaled as $\tilde{\sigma}(\cdot)$ to satisfy Assumption~\ref{asmp:activation} and Assumption~\ref{asmp:cent-activation} w.l.o.g.,
 $$\tilde{\sigma}(\cdot)=\frac{\sigma(\cdot) - a_{0}}{(\sum_{k\geq 1}a_{k}^2)^{\frac{1}{2}}}.$$
 Common activation functions such as ReLU, GELU, Sigmoid, and Swish all live in $L^{2}_{\phi}$. 
\end{remark}

\begin{definition}[Fully-connected MLPs with random weights]
  \label{def:MLP-random-weights}
  Given an activation function $\sigma(\cdot)$, the number of layers $L$, and the input vector $x^{(0)}:= x \in \mathbb{R}^{d_0}$, 
  we define a multi-layer feed-forward neural network which inductively computes the output for each intermediate layer
  \begin{align}
    \label{eq:weight-initialization}
    &x^{(\ell+1)} = \sigma\left(W^{(\ell)} x^{(\ell)}/\| x^{(\ell)} \| \right) \in \mathbb{R}^{d_{l+1}} ~~,~\text{for}~ 0\leq \ell <L \enspace, \text{ with }\\
    \label{eq:weight-initialization2}
    &W^{(\ell)} \in \mathbb{R}^{d_{\ell+1} \times d_{\ell}}, ~~W^{(\ell)} \sim \cM\cN(\mathbf{0}, \mathbf{I}_{d_{\ell+1}} \otimes \mathbf{I}_{d_{\ell}}) \enspace.
  \end{align}
   Here $\mathbf{I}_{d_{\ell}}$ denotes the identity matrix of size $d_{\ell}$, and $\otimes$ denotes the Kronecker product between two matrices. The activation $\sigma(\cdot)$ is applied to each component of the vector input, and the weight matrix $W^{(\ell)}$ in the $\ell$-th layer is sampled from a multivariate Gaussian distribution $\cM\cN(\cdot, \cdot)$. For a vector $v$ and a scalar $s$, the notation $v/s$ denotes the component-wise division of $v$ by scalar $s$. 
\end{definition}

\begin{remark}
We remark that the scaling in \eqref{eq:weight-initialization} matches the standard weight initialization scheme in practice, since in the current setting $\| x^{(\ell)} \|\asymp \sqrt{d_\ell}$ and \eqref{eq:weight-initialization2} is effectively saying that each row $W^{(\ell)}_{i \cdot}/\sqrt{d_\ell} \sim \cN(0, 1/d_{\ell} \cdot \mathbf{I}_{d_{\ell}})$.
\end{remark}

\section{Compositional Kernel and Branching Process}
\label{sec:compo-kernel-branching-proc}

\subsection{Warm up: Duality}
\label{sec:duality}
We start by describing a simple duality between the activation function in multi-layer perceptrons (that satisfies Assumption~\ref{asmp:norm-activation}) and the probability generating function in branching processes. This simple yet essential duality allows us to study deep neural networks, and compare different activation functions borrowing tools from branching processes. This duality in Lemma \ref{lem:key-duality} can be readily established via the Mehler's formula (Proposition~\ref{prop:mehler}). To the best of our knowledge, this probabilistic interpretation (Lemma~\ref{lem:composition-kernel}) is new to the literature. 

\begin{lemma}[Duality: activation and generating functions]
  \label{lem:key-duality}
   Let $\sigma(\cdot)\in L_{\phi}^2(1)$ be an activation function under Assumption~\ref{asmp:norm-activation}, with the corresponding Hermite coefficients $\{a_k\colon k\in\mathbb{Z}^{\geq 0}\}$ satisfying $\sum_{k\geq 0} a_k^2 = 1$. Define the corresponding random variable $Y_\sigma$ with
   \begin{align}
    \PP(Y_\sigma = k) = a_k^2, ~~\forall k \in \mathbb{Z}^{\geq 0} \enspace.
   \end{align}
   We denote the PGF of $Y_\sigma$ (from Definition~\ref{def:PGF}) as $G_{\sigma}(\cdot): [-1,1] \rightarrow [-1,1]$ to be the dual generating function of the activation $\sigma(\cdot)$. Then, for any $x, z \in \mathbb{R}^{d}$, with $\rho := \big\langle x/\| x\| , z/\| z\| \big\rangle \in [-1,1]$, we have
  \begin{equation}
    \EE_{\t \sim \cN(0, \I_d)} \left[ \sigma\left( \t^\top x/\|x\| \right) \sigma\left( \t^\top z/\|z\| \right) \right] = G_\sigma(\rho) \enspace. \label{eq:def-kernel}
  \end{equation}
\end{lemma}

\begin{proof}(sketch) We can rewrite Equation $\ref{eq:def-kernel}$ in terms of the bivariate normal distribution with correlation $\rho$  as $\EE_{(\tilde{x},\tilde{z})\sim \cN_{\rho}}\left[\sigma\left(\tilde{x}\right)\sigma\left(\tilde{z}\right)\right]$. Then, by expanding the density of $\cN_{\rho}$ using Mehler's formula we would retrieve the Hermite coefficients $a_{k}=\EE_{\tilde{x}\sim N(0,1)}[\sigma(\tilde{x})h_{k}(\tilde{x})]$.
\end{proof}

Based on the above Lemma \ref{lem:key-duality}, it is easy to define the compositional kernel associated with a fully-connected MLP with activation $\sigma(\cdot)$. The compositional kernel approach of studying deep neural networks has been proposed in \cite{DFS17, DFGS17}. To see why, let us recall the MLP with random weights defined in Definition~\ref{def:MLP-random-weights}. Then, for any fixed data input $x, z \in \mathbb{R}^{d_0}$, the following holds almost surely for random weights $W^{(\ell)}$
  \begin{align}
    \lim_{d_{\ell+1} \rightarrow \infty}~~\left\langle x^{(\ell+1)}/\| x^{(\ell+1)} \| , z^{(\ell+1)}/\| z^{(\ell+1)} \| \right\rangle = G_{\sigma} \left( \left\langle x^{(\ell)}/\| x^{(\ell) \|} , z^{(\ell)}/\| z^{(\ell)} \| \right\rangle \right)\enspace.
  \end{align}

Motivated by the above equation, one can introduce the asymptotic \textbf{compositional kernel} defined by a deep neural network with activation $\sigma(\cdot)$, in the following way.

\begin{definition}[Compositional kernel]
  \label{def:composition-kernel}
  Let $\sigma(\cdot)\in L_{\phi}^2(1)$ be an activation function that satisfies Assumption~\ref{asmp:norm-activation}.
  Define the $L$-layer compositional kernel $K^{(L)}_{\sigma}(\cdot , \cdot): \mathcal{X} \times \mathcal{X} \rightarrow \mathbb{R}$ to be the (infinite-width) compositional kernel associated with the fully-connected MLPs (from Definition~\ref{def:MLP-random-weights}), such that for any $x, z \in \cX$, we have
  \begin{align}
    K_{\sigma}^{(L)}(x, z) := \underbrace{ G_\sigma \circ \cdots \circ G_{\sigma}}_{\text{composite $L$ times}}\left( \big\langle x/\|x \|, z/\|z\| \big\rangle \right) \enspace.
  \end{align}
  Since the kernel only depends on the inner-product $\rho = \big\langle x/\|x \|, z/\|z\| \big\rangle$, when there is no confusion, we denote for any $\rho \in [-1,1]$
  \begin{align}
    K_{\sigma}^{(L)}(\rho) = \underbrace{ G_\sigma \circ \cdots \circ G_{\sigma}}_{\text{composite $L$ times}}(\rho) \enspace.
  \end{align}
\end{definition}
We will now point out the following connection between the compositional kernel for deep neural networks and the Galton-Watson branching process. Later, we will study the (rescaled) limits, phase transitions, memorization capacity, and spectral decomposition of such compositional kernels.
\begin{lemma}[Duality: MLP and Branching Process]
  \label{lem:composition-kernel}
  Let $\sigma(\cdot)\in L_{\phi}^2(1)$ be an activation function that satisfies Assumption~\ref{asmp:norm-activation}, and $G_\sigma(\cdot)$ be the dual generating function as in Lemma~\ref{lem:key-duality}. Let $\{Z_{\sigma, L}\colon L\in\mathbb{Z}^{\geq 0}\}$ be the Galton-Watson branching process with offspring distribution $Y_\sigma$. Then for any $L\in \mathbb{Z}^{\geq 0}$, the compositional kernel has the following interpretation using the Galton-Watson branching process
  \begin{align}
    K_{\sigma}^{(L)}(\rho) = \EE[ \rho^{Z_{\sigma, L}} ] \enspace.
  \end{align}
\end{lemma}

\begin{proof}(sketch) We prove by induction on $L$  using 
$Z_{\sigma,L}\stackrel{d}{=}\sum_{i=1}^{Z_{\sigma,L-1}}Y_{i,L}$, where $Y_{L,i}\stackrel{i.i.d.}{\sim}Y_{\sigma}^{(L-1)}$.
\end{proof}

The above duality can be extended to study other network architectures. For instance, in the residual network, the duality can be defined as follows: for $x, z \in \mathbb{S}^{d-1}$, $r \in [0,1]$, and a centered activation function $\sigma(\cdot)$ (Assumption \ref{asmp:cent-activation}), define the dual residual network PGF $G_{\sigma}^{\rm res}$ as
\begin{align} \label{eq:resnet}
  G_{\sigma}^{\rm res}(\langle x, z \rangle)& := \EE_{\{\t_j \sim \cN(0,\mathbf{I}_d), j \in [d]\}} \big[ \sum_{j=1}^d \left( \sqrt{1-r} \sigma(\t_j^\top x)/\sqrt{d} +  \sqrt{r} x_j \right)  \left( \sqrt{1-r} \sigma(\t_j^\top z)/\sqrt{d} +  \sqrt{r} z_j \right) \big] \notag\\
  &= (1-r) \EE_{\t \sim \cN(0, \mathbf{I}_d)} \left[  \sigma( \t^\top x) \sigma(\t^\top z)\right] + r \langle x, z \rangle  = (1-r) G^{\rm mlp}_{\sigma}(\langle x, z \rangle) +  r \langle x, z \rangle \enspace.
\end{align}
In the Sections~\ref{sec:asymptotic-limit} and \ref{sec:nonasymptotic-memorization} and later in experiments, we will elaborate on the costs and benefits of adding a linear component to the PGF in the corresponding compositional behavior, both in theory and numerics. The above simple calculation sheds light on why in practice, residual network can tolerate a larger compositional depth. 

\subsection{Limits and phase transitions}
\label{sec:asymptotic-limit}

In this section, we will study the properties of the compositional kernel, in the lens of branching process, utilizing the duality established in the previous section. One important result in branching process is the Kesten-Stigum Theorem \citep{KS66}, which can be employed to assert the rescaled limit and phase transition of the compositional kernel in Theorem \ref{thm:rescaled-limit}.

\begin{theorem}[Rescaled non-trivial limits and phase transitions: compositional kernels]
  \label{thm:rescaled-limit}
  Let $\sigma(\cdot) \in L^2_{\phi}(1)$ be an activation function that satisfies Assumption~\ref{asmp:norm-activation}, with $\{a_k\colon k\in\mathbb{Z}^{\geq 0}\}$ be the corresponding Hermite coefficients that satisfy $\sum_{k\geq 0} a_k^2 = 1$. Define two quantities that depend on $\sigma(\cdot)$,
  \begin{align}
    \mu := \sum_{k\geq 0}a_k^2  k , \quad \mu_\star := \sum_{k>2} a_k^2  k \log k \enspace.
  \end{align}
  Recall the MLP compositional kernel $K_{\sigma}^{(L)}(\cdot): \mathbb{R} \rightarrow \mathbb{R}$ with activation $\sigma(\cdot)$ in Definition~\ref{def:composition-kernel}, and the dual PGF $G_{\sigma}(\cdot)$ in Lemma~\ref{lem:key-duality}. For any $t\geq 0$, the following results hold, depending on the value of $\mu$ and $\mu_\star$:
  \begin{itemize}
    \item[(i)] $\mu \leq 1$. Then, if $a_1^2 \neq 1$,  we have
      \begin{equation}
        \lim_{L \rightarrow \infty} K_{\sigma}^{(L)}(e^{-t}) = 
      1  \quad \text{for $t\geq 0$};
      \end{equation}
      and, if $a_1^2 = 1$, we have $K_{\sigma}^{(L)}(e^{-t}) = e^{-t}$ for all $L \in \mathbb{Z}^{\geq 0}$;
    \item[(ii)]$\mu > 1$ and $\mu_\star < \infty$. Then, there exists $0\leq \xi<1$ with $G_{\sigma}(\xi) = \xi$ and a unique positive random variable $W_{\sigma}$ (that depends on $\sigma$) with a continuous density on $\mathbb{R}^{+}$. And, the non-trivial rescaled limit is
      \begin{align}
        \lim_{L \rightarrow \infty} K_{\sigma}^{(L)}(e^{- t/\mu^L}) = \xi + (1-\xi) \cdot \EE\left[ e^{-t W_{\sigma}} \right] \enspace;
      \end{align}
    \item[(iii)]$\mu > 1$ and $\mu_\star = \infty$. Then, for any positive number $m >0$,  we have 
    \begin{equation}
      \lim_{L \rightarrow \infty} K_{\sigma}^{(L)}(e^{-t/m^{L}}) = 
    \begin{cases}
    1, &\text{ if }~ t  = 0\\
    0, &\text{ if }~ t > 0 
    \end{cases}\enspace .
    \end{equation}
  \end{itemize}
\end{theorem}

The above shows that when looking at the compositional kernel at the rescaled location $e^{-t/\mu^L}$ for a fixed $t>0$, the limit can be characterized by the moment generating function associated with a negative random variable $M_{-W_{\sigma}}(t) = \EE[e^{-t W_{\sigma}}]$ individual to the activation function $\sigma$. The intuition behind such a rescaled location is that the limiting kernels witness an abrupt change of value at $K_{\sigma}^{(L)}(\rho)$ near $\rho=1$ for large $L$ (see the below Corollary~\ref{cor:unscaled-limit}). In the case $\mu>1$, the proper rescaling in Theorem~\ref{thm:rescaled-limit} stretches out the curve and zooms in the narrow window of width $O(1/\mu^L)$ local to $\rho = 1$ to inspect the detailed behavior of the compositional kernel $K_{\sigma}^{(L)}(\rho)$.  Conceptually, the above Theorem classifies the rescaled behavior of the compositional kernel into three phases according to $\mu$ and $\mu_\star$, functionals of the activation $\sigma(\cdot)$. One can also see that the unscaled limit for the compositional kernel has the following simple behavior. 
\begin{corollary}[Unscaled limits and phase transitions]
  \label{cor:unscaled-limit}
  Under the same setting as in Theorem~\ref{thm:rescaled-limit}, the following results hold:
  \begin{itemize}
    \item[(i)] $\mu \leq 1$. Then, for all $\rho \in [0,1]$, if $a_1^2 \neq 1$, we have
    \begin{equation}
      \lim_{L \rightarrow \infty} K_{\sigma}^{(L)}(\rho) = 1 ;
    \end{equation}
    and $K_{\sigma}^{(L)}(\rho) = \rho$ for all $L \in \mathbb{Z}^{\geq0}$ if $a_1^2 = 1$;
    \item[(ii)] $\mu > 1$. Then, there exists a unique $0\leq \xi<1$ with $G_{\sigma}(\xi) = \xi$
        \begin{equation}
          \lim_{L \rightarrow \infty} K_{\sigma}^{(L)}(\rho) = 
        \begin{cases}
        1, &\text{ if }~ \rho  = 1\\
        \xi, &\text{ if }~ \rho \in [0, 1) 
        \end{cases}\enspace .
        \end{equation}
    Under additional assumptions of $G_{\sigma}(\cdot)$ on $(-1,0)$ such as no fixed points or non-negativity, we can extend the above results to $[-1,1]$
    \begin{equation}
          \lim_{L \rightarrow \infty} K_{\sigma}^{(L)}(\rho) = 
        \begin{cases}
        1, &\text{ if }~ \rho  = 1\\
        \xi, &\text{ if }~ \rho \in (-1, 1) 
        \end{cases}\enspace .
        \end{equation}
  \end{itemize}
  Under the additional Assumption \ref{asmp:cent-activation} on $\sigma(\cdot)$, for non-linear activation $\sigma(\cdot)$, we have $\mu > 1$ and $\xi = 0$. Therefore, the unscaled limit for non-linear compositional kernel is
      \begin{equation}
        \lim_{L \rightarrow \infty} K_{\sigma}^{(L)}(\rho) = 
      \begin{cases}
      1, &\text{ if }~ \rho  = 1\\
      0, &\text{if }~ \rho \in (-1, 1) 
      \end{cases}\enspace .
      \end{equation}
\end{corollary}
We remarks that the fact (ii) in the above corollary is not new and has been observed by \cite{DFS17}. On the one hand, they use the fact (ii) to shed light on why more than five consecutive fully connected layers are rare in practical architectures. On the other hand, the phase transition at $\mu = 1$ corresponds to the edge-of-chaos and exponential expressiveness of deep neural networks studied in \cite{NIPS2016_6322}, using physics language.

\section{Memorization Capacity of Compositions: Tradeoffs}
\label{sec:nonasymptotic-memorization}

One advantage of deep neural networks (DNN) is their exceptional data memorization capacity. Empirically, researchers observed that DNNs with large depth and width could memorize large datasets \citep{ZBH+17}, while maintaining good generalization properties. Pioneered by Belkin, a list of recent work contributes to a better understanding of the interpolation regime \citep{BHMM18, BHX19,BRT18,LR18, HMRT19,BLLT19,LRZ20,feldman2019does,nakkiran2020optimal}. 
With the insights gained via branching process, we will investigate the memorization capacity of the compositional kernels corresponding to MLPs, and study the interplay among the \textit{sample size}, \textit{dimensionality}, \textit{properties of the activation}, and the \textit{compositional depth} in a non-asymptotic way.

In this section, we denote $\cX = \{x_i \in \mathbb{S}^{d-1}\colon i\in [n] \}$ as the dataset with each data point lying on the unit sphere. We denote by $\rho:=\max_{i\neq j}|\rho_{ij}|$ with $\rho_{ij}:=\langle x_{i},x_{j}\rangle$ as the maximum absolute value of the pairwise correlations.
Specifically, we consider the following scaling regimes on the sample size $n$ relative to the dimensionality $d$:
\begin{enumerate}
  \item \label{data:r1} Small correlation: We consider the scaling regime $\frac{\log n}{d} < c$ with some small constant $c<1$, where the dataset $\cX$ is generated from a probabilistic model with uniform distribution on the sphere. 
  \item \label{data:r2} Large correlation: We consider the scaling regime $\frac{\log n}{d} > C$ with some large constant $C>1$, where the dataset $\cX$ forms a certain packing set of the sphere. The results also extend to the case of i.i.d. samples with uniform distribution on the sphere.
\end{enumerate}
We name it ``small correlation'' since $\sup_{i\neq j}~|\rho_{ij}|$ can be vanishingly small, in the special case $\frac{\log n}{d} \rightarrow 0$. Similarly, we call it ``large correlation'' as $\sup_{i\neq j}~|\rho_{ij}|$ can be arbitrarily close to $1$, in the special case $\frac{\log n}{d} \rightarrow \infty$.

For the results in this section, we make the Assumptions \ref{asmp:norm-activation} and \ref{asmp:cent-activation} on the activation function $\sigma(\cdot)$, which are guaranteed by a simple rescaling and centering of any $L^2$ activation function.
Let $\K^{(L)} \in \mathbb{R}^{n \times n}$ be the empirical kernel matrix for the compositional kernel at depth $L$, with 
  \begin{align}
    \label{eq:empirical-kernel}
    \K^{(L)}[i,j] = K^{(L)}_{\sigma}\left(\langle x_i, x_j \rangle \right) \enspace.
  \end{align}
For kernel ridge regression, the spectral properties of the empirical kernel matrix affect the memorization capacity: when $\K^{(L)}$ has full rank, the regression function without explicit regularization can interpolate the training dataset. Specifically, the following spectral characterization on the empirical kernel matrix determines the rate of convergence in terms of optimization to the min-norm interpolated solution, thus further determines memorization. The $\kappa$ in following definition of $\kappa$-memorization can be viewed as a surrogate to the condition number of the empirical kernel matrix, as the condition number is bounded by $\frac{1+\kappa}{1-\kappa}$.

\begin{definition} [$\kappa$-memorization]
    We call that a symmetric kernel matrix $\K = [K(x_i, x_j)]_{1\leq i, j\leq n}$ associated with the dataset $\cX=\{x_{i}\}_{i\in [n]}$ has a $\kappa$-memorization property if the eigenvalues $\lambda_{i}(K), i\in[n]$ of $\K$  are well behaved in the following sense
    \begin{align}
    1-\kappa \leq \lambda_i(\K) \leq 1 + \kappa, \enspace i \in [n].
    \end{align}
    We denote by $L_{\kappa}$ the minimum compositional depth such that the empirical kernel matrix $\mathbf{K}$ has the $\kappa$-memorization property.
\end{definition}

\begin{definition} ($\epsilon$-closeness) We say that a kernel matrix $\K = [K(x_i, x_j)]_{1\leq i, j\leq n}$ associated with the dataset $\cX=\{x_{i}\}_{i\in [n]}$ satisfies the $\epsilon$ closeness property if
\begin{align}
\max_{i\neq j} |K(x_{i},x_{j})|\leq \epsilon.
\end{align}
We denote by $\tilde{L}_{\epsilon}$ the minimum compositional depth such that the empirical kernel matrix $\mathbf{K}$ satisfies the $\epsilon$-closeness property.
\end{definition}

We will assume throughout the rest of this section that 
\begin{assumption} (Symmetry of PGF) $|G_{\sigma}(s)|=G_{\sigma}(|s|)$ for all $s\in(-1,1)$.
\end{assumption}

\subsection{Small correlation regime}

To study the small correlation regime, we consider a typical instance of the dataset that are generated i.i.d. from a uniform distribution on the sphere.

\begin{theorem}[Memorization capacity: small correlation regime]
  \label{thm:small-correlation}
  Let  $\cX = \{ x_i \stackrel{\text{i.i.d.}}{\sim} {\rm Unif}(\mathbb{S}^{d-1}) \colon i \in [n] \}$ be a dataset with random instances.
  Consider the regime $\frac{\log n}{d(n)} < c$ with some absolute constant $c<1$ small enough that only depends on the activation $\sigma(\cdot)$. For any $\kappa \in (0, \rho)$,  with probability at least $1-4n^{-1/2}$, the minimum compositional depth $L_\kappa$ to obtain $\kappa$-memorization satisfies
  \begin{align}
   0.5 \cdot \frac{\log\frac{\log{n}}{d}}{\log a_{1}^{-2}} + \frac{\log(0.5 \kappa^{-1})}{\log a_{1}^{-2}}\leq L_{\kappa}\leq \frac{\log\frac{\log{n}}{d}}{\log a_{1}^{-2}} + \frac{2 \log(3 n\kappa^{-1})}{\log a_{1}^{-2}} + 1.
  \end{align}
\end{theorem}

The proof is due to the sharp upper and lower estimates obtained in the following lemma.
\begin{lemma}[$\epsilon$-closeness: small correlation regime]
	\label{thm:clos_small}
	Consider the same setting as in Theorem~\ref{thm:small-correlation}. For any $\epsilon \in (0, \rho)$, with probability at least $1-4n^{-1/2}$, the minimum depth $\tilde L_{\epsilon}$ to obtain $\epsilon$-closeness satisfies
  \begin{align}
 \frac{0.5\log\frac{\log{n}}{d} +\log(0.5 \epsilon^{-1})}{\log a_{1}^{-2}}\leq \tilde{L}_{\epsilon}\leq  \frac{\log\frac{\log{n}}{d} +2 \log(3 \epsilon^{-1})}{\log a_{1}^{-2}} + 1.
  \end{align}
\end{lemma}

In this small correlation regime, Theorem \ref{thm:small-correlation} states that in order for us to memorize a size $n$-dataset $\cX$, the  depth $L$ for the compositional kernel $K_{\sigma}^{(L)}$ scales with three quantities: the linear component in the activation function $a_1^{-2}$, a factor between $\log (\kappa^{-1})$ and $\log (n\kappa^{-1})$, and the logarithm of the regime scaling $\log (\frac{\log n}{d})$. Two remarks are in order. First, as the quantity $\frac{\log n}{d}$ becomes larger, we need a larger depth for the compositional kernel to achieve the same memorization. However, such an effect is mild since the regime scaling enters \textbf{precisely} logarithmically in the form of $\log (\frac{\log n}{d})$. In other words, for i.i.d. data on the unit sphere with $\frac{\log n}{d} \rightarrow 0$, it is indeed easy for a shallow compositional kernel (with depth at most $\log \frac{n \log n}{d}$) to memorize the data. In fact, consider the proportional high dimensional regime $d(n) \asymp n$, then a very shallow network with $1 \precsim L_{\kappa}^{\rm easy} \precsim \log \log n$ is sufficient and necessary to memorize. Second, $a_1^{-2} = 1 + (\sum_{k\geq 2} a_k^2)/a_1^2$ can be interpreted as the amount of non-linearity in the activation function. Therefore, when the non-linear component is larger, we will need fewer compositions for memorization. This explains the necessary large depth of an architecture such as ResNet (Equation \ref{eq:resnet}), where a larger linear component is added in each layer to the corresponding kernel. A simple contrast should be mentioned for comparison: memorization is only possible for linear models when $d \geq n$, whereas with composition and non-linearity, $d \gg \log n$ suffices for good memorization. 

\subsection{Large correlation regime}

To study the large correlation regime, we consider a natural instance of the dataset that falls under such a setting. The construction is based on the sphere packing.
\begin{definition}[r-polarized packing]
  For a compact subset $V\subset \mathbb{R}^{d}$, we say $\cX=\{x_{i}\}_{i\in[n]}\subset V$ is a $r$-polarized packing of $V$ if for all $x_{i}\neq x_{j}\in\cX$, we have $\|x_{i}-x_{j}\|> r$ and  $\|x_{i}+x_{j}\|> r$. We define the polarized packing number of $V$ as $\cP_{r}(V)$, that is
  \begin{align}
  \mathcal{P}_{r}(V)=\max\{ n\colon \text{ exists } r \text{-polarized packing of } V \text{ of size }n\}\enspace.
  \end{align}
\end{definition}

\begin{theorem}[Memorization capacity: large correlation regime]
  \label{thm:large-correlation}
  Let a size-$n$ dataset $\cX = \{ x_i \in \mathbb{S}^{d-1}\colon i\in [n]\}$ be a maximal polarized packing set of the sphere $\mathbb{S}^{d-1}$. Consider the regime $\frac{\log n}{d(n)} > C$ with some absolute constant $C>1$ that only depends on the activation $\sigma(\cdot)$. For any $\kappa \in (0, \min\{ \rho, n s_\star\} )$, the minimum depth $\tilde L_{\epsilon}$ to obtain $\epsilon$-closeness satisfies
  \begin{align}
 1.5 \cdot \frac{\frac{\log n}{d} }{\log \mu} + \frac{\log(0.5 \kappa^{-1}) }{\log a_{1}^{-2}} - 1 \leq 
  L_{\kappa} \leq 3 \cdot \frac{\frac{\log n}{d-1}}{\log \frac{\mu+1}{2}} + \frac{\log(s_\star n\kappa^{-1})}{\log\frac{s_{\star}}{1-(1-s_{\star})\frac{1+\mu}{2}}} + 2.
  \end{align}
   Here $s_\star := \inf\left\{s \in (0,1) ~:~ \frac{1 - G_\sigma(s)}{1-s} \geq \frac{1+\mu}{2}  \right\}$ is a constant that only depends on $\sigma(\cdot)$.
\end{theorem}
\begin{remark} One can carry out an identical analysis in the large correlation regime for the i.i.d. random samples case $x_{i}\sim \mathrm{Unif}(\mathbb{S}^{d-1}), i\in [n]$ with $\frac{\log n}{d}>C$, as in the sphere packing case. Here the constant $C$ only depends on $\sigma(\cdot)$. Exactly the same bounds on $L_{\kappa}$ hold with high probability.
\end{remark}

The proof is due to the sharp upper and lower estimates in the following lemma.
\begin{lemma}[$\epsilon$-closeness: large correlation regime]
	\label{thm:clos_large}
	Consider the same setting as in Theorem~\ref{thm:large-correlation}. For any $\epsilon \in (0, \rho)$, the minimum depth $\tilde L_{\epsilon}$ to obtain $\epsilon$-closeness satisfies
  \begin{align}
  &\tilde{L}_{\epsilon}\geq  \max_{s\in(\epsilon,\rho)} \left\{ \frac{\log(\epsilon^{-1})+\log(s)}{\log a_{1}^{-2}}+\frac{2 \frac{\log n}{d}+\log \left(\frac{1-s}{18.2} \right) }{\log \mu} \right\} - 1,\\
  &\tilde{L}_{\epsilon}\leq  \min_{s\in (\epsilon,\rho)} \left\{ \frac{\log(\epsilon^{-1})+\log(s)}{\log\frac{s}{G_{\sigma}(s)}}+ \frac{ 2\frac{\log n}{d-1}+\log\left(\frac{1-s}{0.06} \right) }{\log \frac{1-G_{\sigma}(s)}{1-s}} \right\} + 2.
  \end{align} 
\end{lemma}

In this large correlation regime, to memorize a dataset, the behavior of the compositional depth is rather different from the small correlation regime. By Theorem~\ref{thm:large-correlation}, we have that the depth $L$ scales with following  quantities: a factor between $\log(\kappa^{-1})$ and $\log(n \kappa^{-1})$ same as before, the regime scaling $\frac{\log n}{d}$, and functionals of the activation $\mu$, $a_{1}^{-2}$, and $s_{\star}$. Few remarks are in order. First, in this large correlation regime $n \gg \exp(d)$, memorization is indeed possible. However, the compositional depth needed increases \textbf{precisely} linearly as a function of the regime scaling $\frac{\log n}{d}$. The above is in contrast to the small correlation regime, where the dependence on the regime scaling is logarithmic as $\log (\frac{\log n}{d})$. For hard dataset instances on the sphere with $\frac{\log n}{d} \rightarrow \infty$, one needs at least $\frac{\log n}{d}$ depth for the compositional kernel to achieve memorization. In fact, consider the fixed dimensional regime with $d(n) \asymp 1$, then a deeper network with depth $L_{\kappa}^{\rm hard} \asymp \log n$ is sufficient and necessary to memorize, which is much larger than the depth needed in the proportional high dimensional regime with $L_{\kappa}^{\rm easy} \precsim \log \log n$.
Second, for larger values of $a_{1}^{-2}$ and $\mu$ we will need less compositional depths, as the amount of non-linearity is larger. To sum up, with non-linearity and composition, even in the $d \ll \log n$ regime with a hard data instance, memorization is possible but with a deep neural network.

\section{Spectral Decomposition of Compositional Kernels}
\label{sec:spherical-harmonics}

In this section, we investigate the spectral decomposition of the compositional kernel function. We study the case where the base measure is a uniform distribution on the unit sphere, denoted by $\tau_{d-1}$. Let $S_{d-1}$ be the surface area of $\mathbb{S}^{d-1}$. To state the results, we will need some background on the spherical harmonics. We consider the dimension $d\geq2$, and will use $k \in \mathbb{Z}^{\geq 0}$ to denote an integer.
\begin{definition}[Spherical harmonics, Chapter 2.8.4 in \cite{AH12}]
  \label{def:spherical-harmonics}
  Let $\mathbb{Y}_{k}^d$ be the space of $k$-th degree spherical harmonics in dimension $d$, and let 
  $\{ Y_{k, j}(x) \colon j\in[N_{k,d}]\}$ be an orthonormal basis for $\mathbb{Y}_{k}^d$ , with
  \begin{align}
    \int  Y_{k, i}(x) Y_{k, j}(x) d \tau_{d-1} (x) = \mathbbm{1}_{i = j} \enspace.
  \end{align}
  Then, the sets form an orthogonal basis for the space $L^2_{\tau_{d-1}}$ of $L^2$-integrable functions on $\mathbb{S}^{d-1}$ with the base measure $\tau_{d-1}$, noted below
  \begin{align}
    L^2_{\tau_{d-1}} = \bigoplus_{k=0}^{\infty}\mathbb{Y}_{k}^d\enspace.
  \end{align}
  Moreover, the dimensionality ${\rm dim} \mathbb{Y}_{k}^d = N_{k, d}$ are the coefficients of the generating function
  \begin{align}
    \label{eq:N-k-d}
    \sum_{k=0}^\infty N_{k, d} s^k = \frac{1+s}{(1-s)^{d-1}}, ~|s| <1 \enspace.
  \end{align}
\end{definition}

\begin{definition}[Legendre polynomial, Chapter 2.7 in \cite{AH12}]
  \label{def:legendre-poly}
  Define the Legendre polynomial of degree $k$ with dimension $d$ to be
  \begin{align}
    P_{k,d}(t) = (-1)^k \frac{\Gamma(\frac{d-1}{2})}{2^k\Gamma(k+\frac{d-1}{2})} (1-t^2)^{-\frac{d-3}{2}} \left( \frac{d}{dt} \right)^k (1-t^2)^{k + \frac{d-3}{2}} \enspace.
  \end{align}
  The following orthogonality holds
  \begin{equation}
    \label{eq:legendre-ortho}
    \int_{-1}^{1} P_{k,d}(t) P_{\ell, d}(t) (1-t^2)^{\frac{d-3}{2}} dt = 
    \begin{cases}
      0, &\text{ if } k\neq \ell\\
      \frac{S_{d-1}}{N_{k, d}S_{d-2}}, &\text{ if } k = \ell
    \end{cases} \enspace.
  \end{equation}
\end{definition}

Recall that the compositional kernel $K_{\sigma}^{(L)}(\cdot)$ and the random variable $Z_{\sigma, L}$, which denotes the size of the $L$-th generation, as in Lemma~\ref{lem:composition-kernel}.
Then, we have the following theorem describing the spectral decomposition of the compositional kernel function and the associated integral operator.

\begin{theorem}[Spectral decomposition of compositional kernel]
  \label{thm:spectral-decomposition}
  Consider any $x, z \in \mathbb{S}^{d-1}$. Then, the following spectral decomposition holds for the compositional kernel $K_{\sigma}^{(L)}$ with any fixed depth $L \in \mathbb{Z}^{\geq 0}$: 
   \begin{align}
    K_{\sigma}^{(L)} \big( \langle x, z\rangle \big) = \sum_{k = 0}^\infty \lambda_k \sum_{j=1}^{N_{k,d}} Y_{k,j}(x) Y_{k,j}(z)\enspace ,
   \end{align}
   where the eigenfunctions $Y_{k,j}(\cdot)$ form an orthogonal basis of $L^2_{\tau_{d-1}}$, and the eigenvalues  $\lambda_{k}$ satisfy the following formula
   \begin{align}
    \lambda_k := \frac{S_{d-2}}{S_{d-1}} \Gamma(\frac{d-1}{2}) \sum_{\ell \geq 0} \PP(Z_{\sigma, L} = k + \ell) \frac{(k+\ell)!}{\ell!} \frac{1+(-1)^{\ell}}{2^{k+1}} \frac{\Gamma(\frac{\ell+1}{2}) }{\Gamma(k+\frac{\ell+d}{2})}  \enspace.
   \end{align}
\end{theorem}

The associated integral operator $\mathcal{T}_{\sigma}^{(L)}:L^2_{\tau_{d-1}} \rightarrow L^2_{\tau_{d-1}}$ with respect to the kernel $K_{\sigma}^{(L)}$ is defined as
\begin{align}
  \big( \mathcal{T}_{\sigma}^{(L)} f \big) (x) := \int K_{\sigma}^{(L)} \big( \langle x, z\rangle \big) f(z)  d \tau_{d-1}(z)\enspace \text{ for any $f\in L^2_{\tau_{d-1}}$} \enspace.
\end{align}
From Theorem~\ref{thm:spectral-decomposition}, we know that the eigenfunctions of the operator $\mathcal{T}_{\sigma}^{(L)}$ are the spherical harmonic basis $\{ Y_{k,j}: k\in\mathbb{Z}^{\geq 0}, j \in [N_{k,d}] \}$, with $N_{k,d}$ identical eigenvalues $\lambda_k$ such that
\begin{align}
  \mathcal{T}_{\sigma}^{(L)} Y_{k,j} = \lambda_k Y_{k,j} \enspace.
\end{align}
The above spectral decompositions are important because it helps us study the generalization error (in the fixed dimensional setting) of regression methods with the compositional kernel $K_{\sigma}^{(L)}$. More specifically, understanding the eigenvalues of the compositional kernels means that we can employ the classical theory on reproducing kernel Hilbert spaces regression \citep{CV06} to quantify generalization error, when the dimension is fixed. In the case when dimensionality grows with the sample size, several attempts have been made to understand the generalization properties of the inner-product kernels \citep{LR18,LRZ20} in the interpolation regime, which includes these compositional kernels as special cases.

\section{Kernels to Activations: New Random Features Algorithms}
\label{sec:new-random-feature-alg}

Given any $L^2_{\phi}$ activation function $\sigma(\cdot)$ (as in Definition \ref{asmp:activation}), we can define a sequence of positive definite (PD) compositional kernels $K_{\sigma}^{(L)}$, with $L\geq 0$, whose spectral properties have been studied in the previous section, utilizing the duality established in Section~\ref{sec:duality}. Such compositional kernels are non-linear functions on the inner-product $\langle x, z\rangle$ (rotation-invariant), and we will call them the \textit{inner-product kernels} \citep{KK12}. In this section, we will investigate the \textit{converse question}: given an arbitrary PD inner-product kernel, can we identify an activation function associated with it? We will provide a positive answer in this section. Direct algorithmic implications are new random features algorithms that are distinct from the well-known random Fourier features and random kitchen sinks algorithms studied in \cite{RR08,RR09}.

Define an inner-product kernel $K(\cdot, \cdot): \mathbb{S}^{d-1} \times \mathbb{S}^{d-1}$, with $d \geq 2$,
\begin{align}
  \label{eq:inner-product-kernel}
  K(x, z) := f( \langle x, z \rangle) \enspace,
\end{align}
where $f\colon [-1, 1] \rightarrow \mathbb{R}$ is a continuous function. Denote the expansion of $f$ under the Legendre polynomials $P_{k, d}$ (see Definition~\ref{def:legendre-poly}) as
  \begin{align}
    \label{eq:legendre-series}
    f(t) = \sum_{k = 0}^{\infty} \alpha_k P_{k, d}(t) \enspace.
  \end{align}
To define the activations corresponding to an arbitrary PD inner-product kernel, we require the following theorem due to \cite{Sch42}. 
\begin{proposition}[Theorem 1 in \cite{Sch42}]
  \label{thm:sch42}
   For a fixed $d\geq 2$, the inner-product kernel $K(\cdot, \cdot)$ in \eqref{eq:inner-product-kernel} is positive definite on $\mathbb{S}^{d-1} \times \mathbb{S}^{d-1}$ if and only if $\alpha_k \geq 0$ for all $k \in \mathbb{Z}^{\geq 0}$ in Equation~\eqref{eq:legendre-series}.
\end{proposition}

Now, we are ready to state the activation function $\sigma_{f}(\cdot)$ defined based on the inner-product kernel function $f(\cdot)$. 
\begin{theorem}[Kernels to activations]
  \label{thm:kernels-to-activations}
  Consider any positive definite inner product kernel $K(x, z) := f( \langle x, z \rangle)$ on $\mathbb{S}^{d-1} \times \mathbb{S}^{d-1}$ associated with the continuous function $f: [-1, 1] \to \mathbb{R}$. Assume without loss of generality that $f(1)=1$, and recall the definition of $N_{k,d}$ in~\eqref{eq:N-k-d}. Due to Proposition~\ref{thm:sch42}, the Legendre coefficients $\{ \alpha_k\colon k\in\mathbb{Z}^{\geq 0} \}$, defined in Equation \eqref{eq:legendre-series}, of $f(\cdot)$  are non-negative. 
  
  One can define the following dual activation function $\sigma_f: [-1,1] \rightarrow \mathbb{R}$ 
  \begin{align}
    \label{eq:dual-activation}
    \sigma_f(t) = \sum_{k=0}^{\infty} \sqrt{\alpha_k N_{k,d}} P_{k,d}(t) \enspace.
  \end{align}
  Then, the following statements hold:
  \begin{itemize}
    \item[(i)] $\sigma_{f}(\cdot)$ is $L^2$ in the following sense
    $\int_{-1}^1 \left(\sigma_f(t)\right)^2  \frac{S_{d-2}}{S_{d-1}} (1- t^2)^{\frac{d-3}{2}} dt = 1 \enspace.$
    \item[(ii)] For any $x, z \in \mathbb{S}^{d-1}$
  \begin{align}
    \EE_{\mathbf{\xi} \sim \tau_{d-1}} \left[ \sigma_f( \mathbf{\xi}^\top x) \sigma_f(\mathbf{\xi}^\top z) \right] = f( \langle x, z \rangle) = K(x, z) \enspace,
  \end{align}
  where $\xi$ is sampled from a uniform distribution on the sphere $\mathbb{S}^{d-1}$.
  \end{itemize}
\end{theorem}

The above theorem naturally induces a new random features algorithm for kernel ridge regression, described below. Note that the kernel $K$ can be any compositional kernel, which is positive definite and of an inner-product form.

\begin{algorithm}[H]
\SetAlgoLined
\KwResult{Given a normalized dataset $\cX = \{x_i \in \mathbb{S}^{d-1}, i \in [n]\}$, an integer $m$, and a positive definite inner-product kernel $K(x, z) = f( \langle x, z \rangle)$, return a randomized non-linear feature matrix $\mathbf{\Phi} \in \mathbb{R}^{n \times m}$ for kernel ridge regression.}
 {\bf Step 1}: Calculate the Legendre coefficients of $f(t) = \sum_{k=0}^\infty \alpha_k P_{k, d}(t)$\;
 {\bf Step 2}: Obtain the dual activation function $\sigma_f(t) =  \sum_{k=0}^{\infty} \sqrt{\alpha_k N_{k,d}} P_{k,d}(t)$,  $t\in [-1,1]$\;
 {\bf Step 3}: Sample $m$-i.i.d. isotropic Gaussian vectors $t_j \sim \cN(0, \mathbf{I}_d)$, and then define the non-linear feature matrix
 \begin{align}
  \mathbf{\Phi}[i,j] = \sigma_f\left( \langle  x_i, \t_j/\| \t_j \| \rangle \right), ~i\in [n], j\in [m] \enspace.
 \end{align}
 The feature matrix satisfies $\lim_{m\rightarrow\infty} \left(1/m\cdot  \mathbf{\Phi} \mathbf{\Phi}^\top\right) [i, \ell] = f(\langle x_i, x_\ell \rangle)$ for all $i,\ell\in [n]$ a.s.. 
 \caption{New random features algorithm for positive definite inner-product kernels $K(x, z) = f( \langle x, z \rangle)$ based on Theorem~\ref{thm:kernels-to-activations}.}
 \label{alg:SH}
\end{algorithm}

It is clear from Theorem~\ref{thm:spectral-decomposition} that all compositional kernels $K_{\sigma}^{(L)}$ are positive definite, though the converse statement is not true. A notable example is the kernel $P_{k,d}(\langle x, z \rangle): \mathbb{S}^{d-1} \times \mathbb{S}^{d-1} \rightarrow \mathbb{R}$, which is PD kernel but with negative Taylor coefficients, thus cannot be a compositional kernel. For the special case of compositional kernels with depth $L$ and activation $\sigma(\cdot)$, it turns out one can define a new "compressed" activation $\sigma^{(L)}(\cdot) \in L^2_\phi$ to represent the depth-$L$ compositional kernel. We propose the following algorithm:

\begin{algorithm}[H]
\SetAlgoLined
\KwResult{Given a normalized dataset $\cX = \{x_i \in \mathbb{S}^{d-1}, i \in [n]\}$, an integer $m$, and a compositional kernel $K_{\sigma}^{(L)}(x, z) = K_{\sigma}^{(L)}( \langle x, z \rangle)$, return a randomized non-linear feature matrix $\mathbf{\Psi} \in \mathbb{R}^{n \times m}$ for kernel ridge regression.}
 {\bf Step 1}: Calculate the Taylor coefficients of $ K_{\sigma}^{(L)}(t) = \sum_{k=0}^\infty \alpha_k t^k$ with $\alpha_k = \PP(Z_{\sigma, L} = k)$\;
 {\bf Step 2}: Obtain the "compressed" activation function 
 \begin{align}
   \label{eq:compressed-act}
  \sigma^{(L)}(t) =  \sum_{k=0}^{\infty} \sqrt{\alpha_k} h_k(t), ~t\in \mathbb{R} \enspace;
 \end{align}
 {\bf Step 3}: Sample $m$-i.i.d. isotropic Gaussian vectors $\t_j \sim \cN(0, \mathbf{I}_d)$, and then define the non-linear feature matrix
 \begin{align}
  \mathbf{\Psi}[i,j] = \sigma^{(L)}\left( \langle  x_i, \t_j \rangle \right), ~i\in [n], j\in [m] \enspace. \label{eq:random-features}
 \end{align}
  The feature matrix satisfies $\lim_{m\rightarrow\infty} \left(1/m\cdot  \mathbf{\Psi} \mathbf{\Psi}^\top\right) [i, \ell] = K_\sigma^{(L)}(\langle x_i, x_\ell \rangle)$ for all $i,\ell\in [n]$ a.s..
 \caption{New random features algorithm for compositional kernels $K_{\sigma}^{(L)}$.}
 \label{alg:HE}
\end{algorithm}

First, let us discuss the relationship between our random features algorithms above and that in \cite{RR08, RR09}. \cite{RR08} employed the duality between the PD shift-invariant kernel $k(x-z)$ and a positive measure that corresponds to the inverse Fourier transform of $k$: the random features are constructed based on the specific positive measure where sampling could be a non-trivial task. In contrast, we utilize the duality between the PD inner-product kernel and an activation function, where the random features are always generated based on the uniform distribution on the sphere (or the isotropic Gaussian), but with different activations $\sigma(\cdot)$. It is clear that sampling uniformly from the sphere can be easily done by sampling $\t \sim \cN(0, I_d)$, and returning $\t/\| \t\|$.

We now conclude this section by stating another property of Algorithm~\ref{alg:HE}, implied by the Mehler's formula. It turns out that under a certain high dimensional scaling regime, say $d(n) \asymp n^{\frac{2}{\iota+1}}$
for some fixed integer $\iota \in \mathbb{Z}^{>0}$, the compressed activation $\sigma^{(L)}$ in \eqref{eq:compressed-act} can be truncated without loss using only the low degree components $k \leq \iota$. For notation simplicity, in the statement below, we drop the superscript $(L)$ in the kernel and the compressed activation.
\begin{theorem}[Empirical kernel matrix: truncation and implicit regularization]
  \label{thm:ekm-truncation-regularization}
  Consider the compositional kernel function $K_{\sigma}(t) = \sum_{k=0}^\infty \alpha_k t^k$ as in Algorithm~\ref{alg:HE}, and the dataset $\mathbf{X} = [x_1^\top, \cdots, x_n^\top]^{\top} \in \mathbb{R}^{n \times d}$ with column $x_i \stackrel{i.i.d.}{\sim} {\rm Unif}(\mathbb{S}^{d-1})$. Consider the high dimensional regime where the dimensionality $d(n)$ scales with $n$, satisfying $$d(n) \succsim n^{\frac{2}{\iota+1} + \delta},$$ for some fixed integer $\iota \in \mathbb{Z}^{>0}$ and any fixed small $\delta>0$. Define a truncated activation based on \eqref{eq:compressed-act}, at degree level $\iota$
  \begin{align}
    \sigma_{\leq \iota}(t) :=  \sum_{k=0}^{\iota} \sqrt{\alpha_k} h_k(t) \enspace.
  \end{align}
Then the empirical kernel matrix $\K \in \mathbb{R}^{n \times n}$ satisfies the following decomposition
  \begin{align}
    \K = \underbrace{\big( \sum_{k > \iota} \alpha_k \big) \cdot \I_n}_{\text{implicit regularization}} + \underbrace{\EE_{\t \sim \cN(0, \I_d)}\big[ \sigma_{\leq \iota}(\mathbf{X} \t) \sigma_{\leq \iota}(\mathbf{X} \t)^\top \big]}_{\text{degree truncated random features}} + \underbrace{\mathbf{R}}_{\text{remainder}}
  \end{align}
  with the remainder matrix $\mathbf{R} \in \mathbb{R}^{n\times n}$ satisfies
  \begin{align}
    \| \mathbf{R} \|_{\rm op} \rightarrow 0,~~\text{as $n, d(n) \rightarrow \infty$.}
  \end{align}
\end{theorem}
A direct consequence of the above theorem is that, the empirical kernel matrix $\K$ shares the same eigenvalues and empirical spectral density as the matrix $( \sum_{k > \iota} \alpha_k ) \cdot \I_n + \EE_{\t \sim \cN(0, \I_d)}[ \sigma_{\leq \iota}(\mathbf{X} \t) \sigma_{\leq \iota}(\mathbf{X} \t)^\top ]$, asymptotically. The implicit regularization matrix is contributed by the high degree components of the activation function collectively. Therefore, in Algorithm~\ref{alg:HE} with truncation level $\iota$ the random features in Equation \ref{eq:random-features} are generated according to 
 \begin{align}
 \mathbf{\Psi}[i,j] = \sigma^{(L)}_{\leq \iota}\left( \langle  x_i, \t_j \rangle \right) +\big(1-\sum_{k=0}^{\iota}\alpha_{k}\big)^{1/2} \cdot z_{ij}, ~i\in [n], j\in [m]\enspace,
 \end{align}
where $z_{ij}\stackrel{i.i.d.}{\sim} \cN(0,1)$ are Gaussian noise.

\section{Numerical Investigation}

In this section, we will study numerically the theory established in the previous sections. We will experiment with four common activations used in practice as a proof of concept, namely ReLU, GeLU, Swish, and Sigmoid, and four PGFs associated with non-negative discrete probability distributions including Poisson$(\lambda)$, Binomial($n$,$p$), Geometric($p$), and Uniform($0, n$). To execute the theory numerically, we introduce a simple and general Algorithm $\ref{alg:coef_HE}$ for estimating the Hermite coefficients of \textit{any} activation function $\sigma(\cdot)$, with provable guarantees. The numerical stability of this algorithm to estimate the Hermite coefficients can be seen in Figure
\ref{fig:alg_he}. The truncation level considered is $\iota=20$. We will use this level throughout the rest of the experiments in this section.

\begin{algorithm}[!htbp]
\KwResult{Given an activation function $\sigma(t)=\sum_{k=0}^{\infty}a_{k} h_{k}(t) \in L^2_\phi$ and a truncation level $\iota$, return an estimate of the Hermite coefficients $\{a_{k}\}_{k=0}^{\iota}$.}
{\bf Step 1}: Generate $M$ inputs $\{x_{i}\}_{i=1}^{M}$ from a standard normal $\cN(0,1)$\;
{\bf Step 2}: For each $k =0,1,2,\ldots, \iota$, calculate $$\hat{a}_k = \frac{1}{M} \sum_{i=1}^M \sigma(x_i) h_k(x_i)\enspace.$$ The coefficients satisfy $\lim_{M\to \infty}\hat{a}_{k}=a_{k}$ a.s. for all $k$.
\caption{Estimating the Hermite coefficients of an activation $\sigma(\cdot)$.}
\label{alg:coef_HE}
\end{algorithm}

\begin{table}[!htbp]
  \begin{center}
    \begin{tabular}{c|c|c|c|c}
    \toprule
    Activations &ReLU & GeLU & Sigmoid & Swish  \\ \hline
    $\sigma(t)$ &  $\max{(0,t)}$ & $t\cdot \Phi(t)$ & $1/(1+e^{-t})$ &$t/(1+e^{-t})$ \\

    \midrule

    PGFs &Poisson($\lambda$) & Binomial($n$,$p$) & Geometric($p$)&  Uniform($0$,$n$)  \\ \hline
    $G(s)$ &  $\exp\{\lambda(s-1)\}$ & $(1-p + ps)^{n}$ & $(1-p)/(1-ps)$ &$(1-s^{n+1})/(n(1-s))$ \\ 
    \bottomrule
    \end{tabular}
  \end{center}
  \caption{Examples of common activation functions and PGFs.}
  \label{tbl:examples}
\end{table}

\begin{figure}[!htbp]
  \begin{center}
  \includegraphics[scale=0.6]{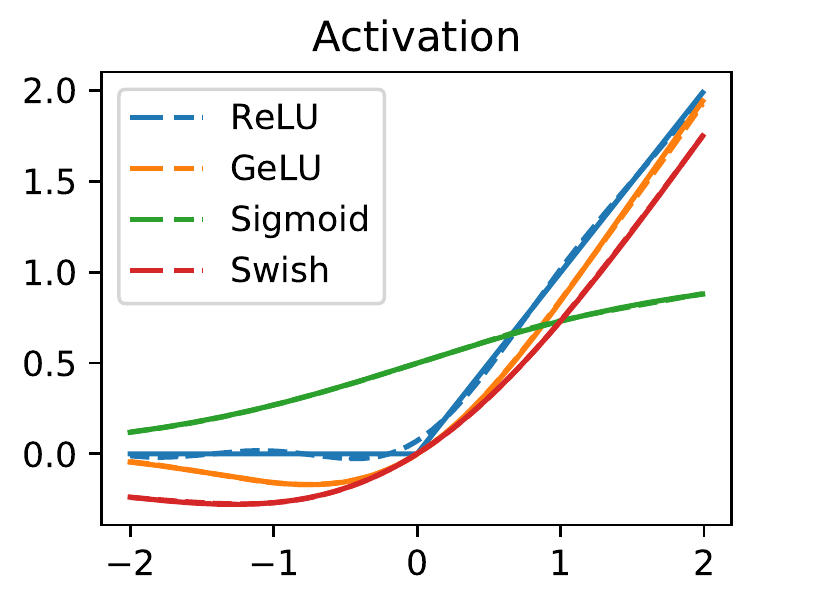}
  \end{center}
  \label{fig:alg_he}
  \caption{Plot of the numerical stability of Algorithm $\ref{alg:coef_HE}$ in approximating activation functions. Solid lines represent the activations, and the dashed lines represent the approximations of the activation functions.}
\end{figure}

\subsection{Duality: Activations and PGFs}
According to Lemma \ref{lem:key-duality},  there is a duality between activation functions $\sigma(\cdot)$ (normalized  as in \ref{asmp:norm-activation}) and PGFs $G_{\sigma}(\cdot)$ (of a discrete non-negative probability distribution). In the first two plots in Figure \ref{fig:duality}, we start from a probability distribution, and construct its corresponding activation function. Reversely, in the last two plots in Figure \ref{fig:duality}, we start from an activation function, and approximate (using Algorithm \ref{alg:coef_HE}) its corresponding PGF.   

\begin{figure}[!htbp]
  \begin{center}
  \includegraphics[scale=0.55]{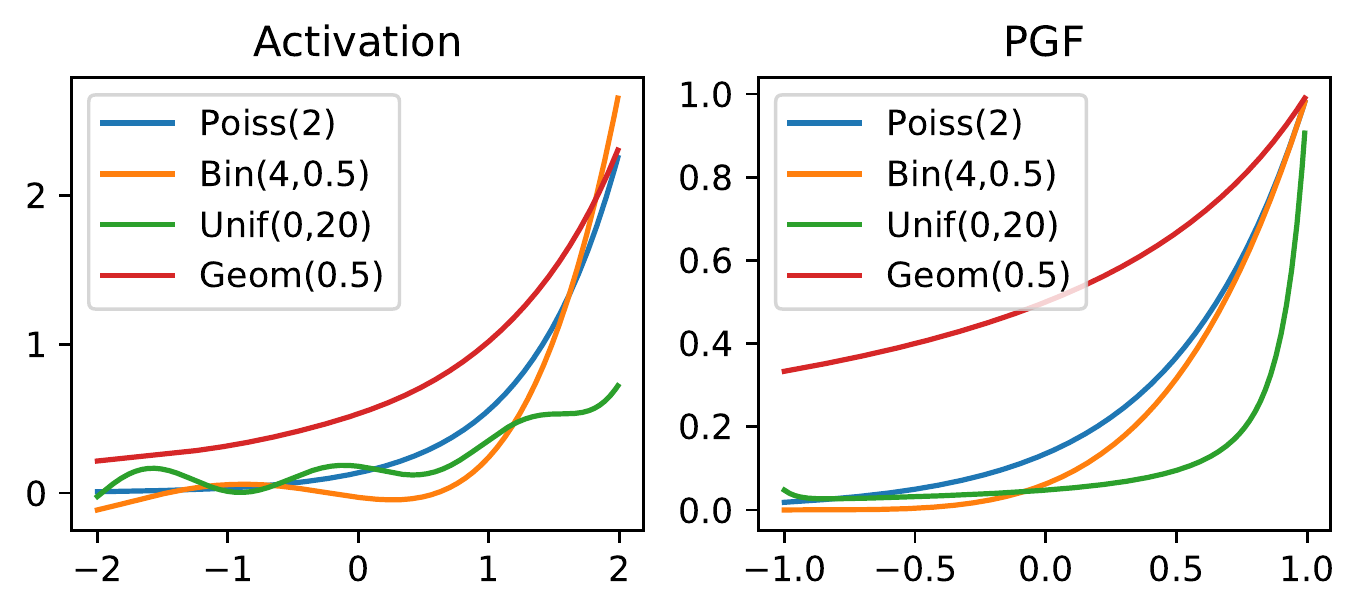}
  \includegraphics[scale=0.55]{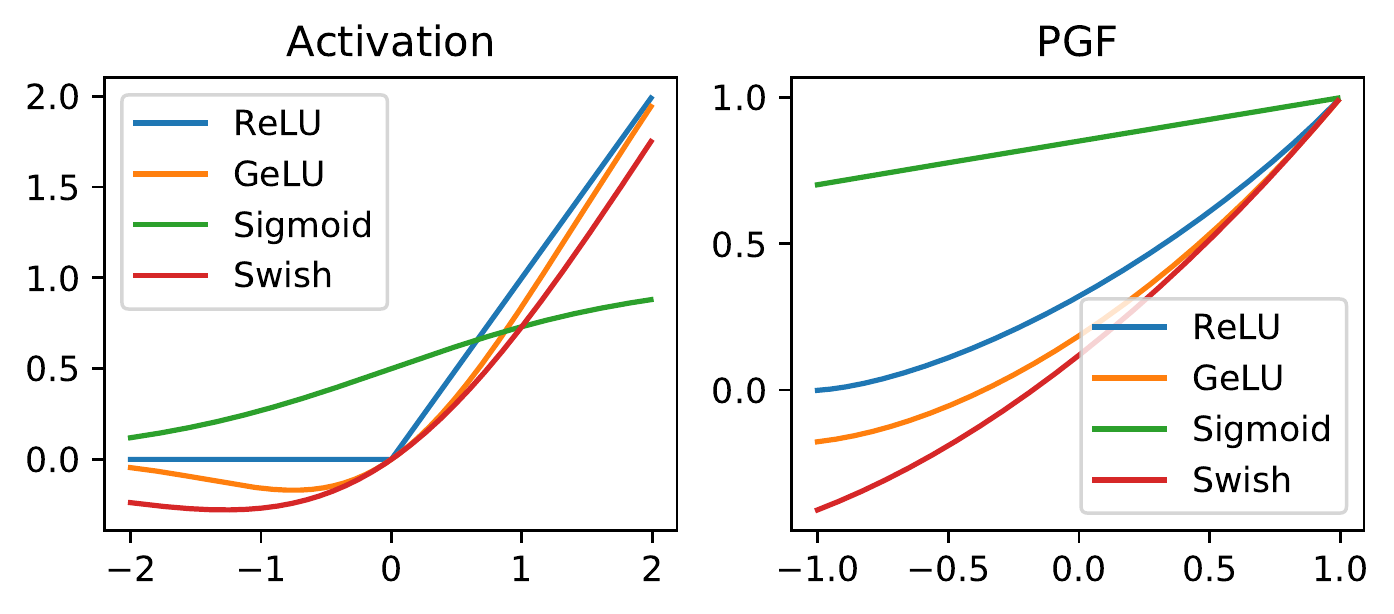}
  \end{center}
  \caption{Duality between activations and PGFs.}
  \label{fig:duality}
\end{figure}

\subsection{Kernel limits}

In this section, we will illustrate the compositional behavior of the kernels, both in the unscaled (Corollary~\ref{cor:unscaled-limit}) and rescaled (Theorem \ref{thm:rescaled-limit}) cases.
We start with Figure \ref{fig:comp-unscaled} on the compositional behavior of unscaled kernels $K^{(L)}_\sigma(\rho)$ as a function of $\rho\in [-1,1]$, without centering the activation functions. According to Corollary \ref{cor:unscaled-limit}, we know that the unscaled limits of compositional kernels are determined strictly by a phase transition of $\mu$ at $1$. On the interval $(-1,1)$, ReLU and Sigmoid's composite kernels converge to $1$, while GeLU and Swish's converge to their corresponding extinction probability. In Figure \ref{fig:comp-unscaled-centered} we plot the compositional behavior of the $K^{(L)}_\sigma(\rho)$ for centered activations $\sigma$: For centered ReLU, GeLU, Swish, the limit approaches $0$ for $\rho \in [-1,1)$, and approaches $1$ at $\rho=1$. For centered Sigmoid, $a_1 \approx 1$, which explains the resemblance to the linear kernel.

On the other hand, for rescaled kernels $K^{(L)}_\sigma(e^{-t/\mu^L})$ as a function of $t\in [0,\infty)$, non-trivial limits that depend on $\sigma(\cdot)$ exist, when $\mu_\star < \infty$. In the un-centered and rescaled case (Theorem \ref{thm:rescaled-limit}), we have that ReLU and Sigmoid's compositional kernels converge to $1$, while GeLU and Swish's ones approach a non-trivial limit, shown in Figure~\ref{fig:comp-rescaled}. If we center the activations, all rescaled kernels will approach non-trivial limits as seen in Figure \ref{fig:comp-rescaled-centered}.

In terms of the convergence speed of kernels to their unscaled limit, Theorem \ref{thm:small-correlation} explains how fast the curve flattens around $0$, and Theorem \ref{thm:large-correlation} on the rate it flattens around $1$. The convergence speed, in fact, determines the memorization capacity of composition kernels. To flatten around $0$, we need the number of compositions to scale with $\frac{1}{\log a_{1}^{-2}}$, while to flatten around $1$, we need the compositional depth to scale with $\frac{1}{\mu-1}$. For example, Sigmoid has $\mu\approx 1.03$ and $a_{1}^2\approx 0.99$, thus explaining slow convergence compared to the other activations. Table~\ref{tbl:val-mu} summarizes the crucial quantities that determine the compositional behavior for each activation.

\begin{table}[!htbp]
  \begin{center}
    \begin{tabular}{c|c c|c c|c c|c c}
    \toprule
    Activation  &  \multicolumn{2}{c|}{ReLU} & \multicolumn{2}{c|}{GeLU}&\multicolumn{2}{c|}{Sigmoid}&\multicolumn{2}{c}{Swish}\\
    \midrule
    $\mu$ & $0.95$ & $1.39$ & $1.08$ & $1.47$ & $0.15$ & $1.03$ & $1.07$ & $1.22$\\
    $\mu_{\star}$ & $0.48$ & $0.69$ & $0.39$ & $0.89$ & $0.01$ & $0.05$ & $0.28$ & $0.31$\\
    $a_{1}^2$ & $0.50$ & $0.74$ & $0.59$ & $0.71$ & $0.15$ & $0.99$ & $0.80$  & $0.70$\\
    $\xi$ & $1.00$ & $0.00$ & $0.76$& $0.00$ &$1.00$ &$0.00$ & $0.66$ &$0.00$  \\
    \bottomrule
\end{tabular}
  \end{center}
  \caption{Approximations of $\mu$, $\mu_{\star}$ (from Theorem \ref{thm:rescaled-limit}), and $a_{1}$ (from Definition \ref{asmp:activation}). For each activation, there are two columns: values for un-centered activation (left) and values for centered activation (right). }
  \label{tbl:val-mu}
\end{table}

\begin{figure}[!htbp]
  \begin{center}
  \includegraphics[scale=0.6]{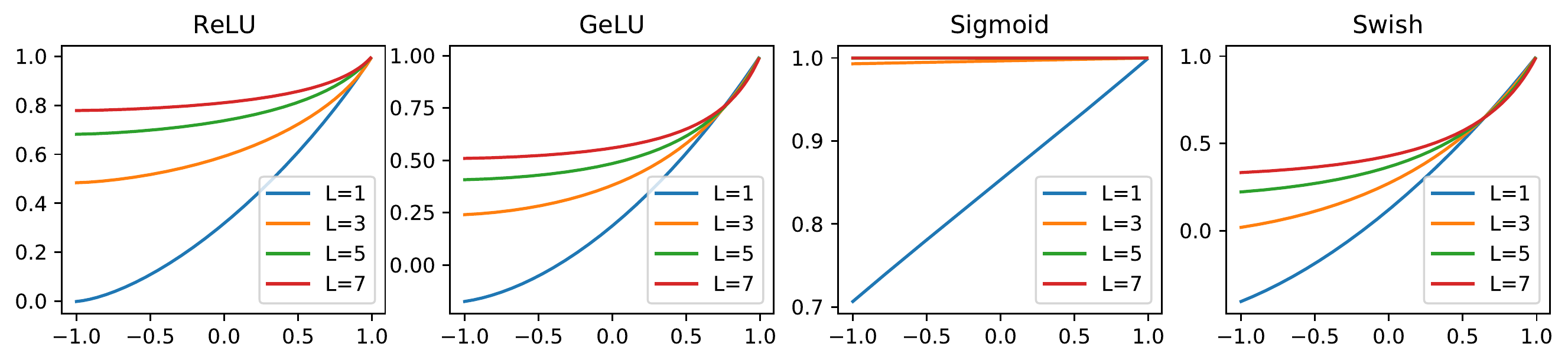}
  \end{center}
  \caption{Compositional behavior of unscaled kernels.}
  \label{fig:comp-unscaled}
\end{figure}

\begin{figure}[!htbp]
  \begin{center}
  \includegraphics[scale=0.6]{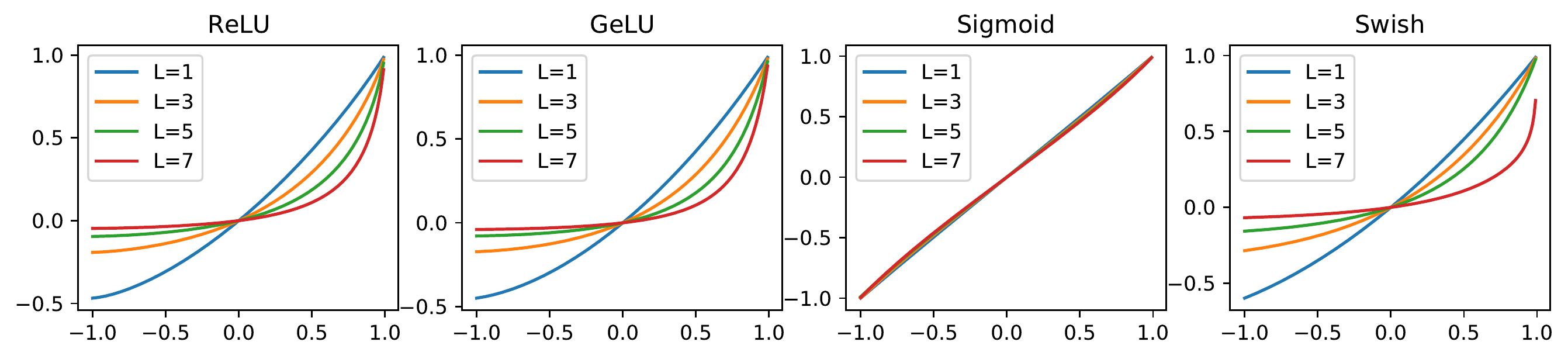}
  \end{center}
  \caption{Compositional behavior of unscaled kernels with centered activations.}
  \label{fig:comp-unscaled-centered}
\end{figure}

\begin{figure}[!htbp]
  \begin{center}
  \includegraphics[scale=0.6]{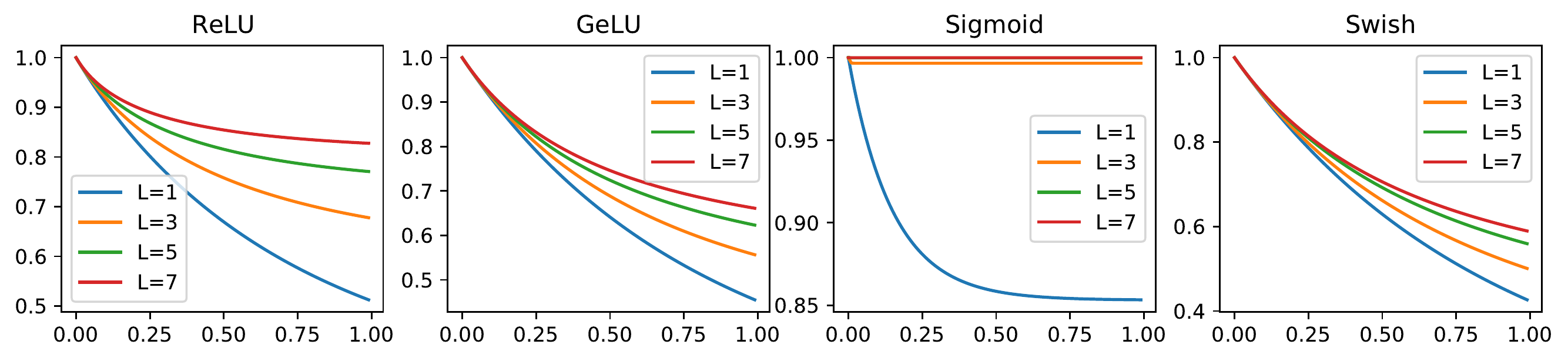}
  \end{center}
  \caption{Compositional behavior of rescaled kernels.}
  \label{fig:comp-rescaled}
\end{figure}

\begin{figure}[!htbp]
  \begin{center}
  \includegraphics[scale=0.6]{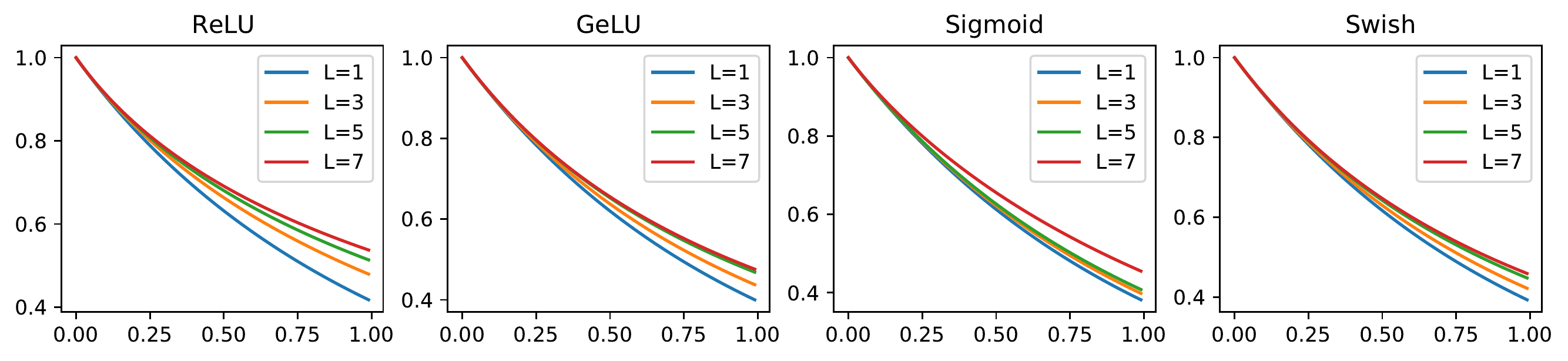}
  \end{center}
  \caption{Compositional behavior of rescaled kernels with centered activations.}
  \label{fig:comp-rescaled-centered}
\end{figure}

\subsection{Applications to datasets: new random features algorithm}
In this section, we will investigate the "compressed" activations obtained from compositional kernels to generate random features, as in Algorithm~\ref{alg:HE}. In Figure \ref{fig:compressed-act}, we plot the shape of the compressed activations, where the initial activations were re-centered and re-scaled (Assumption 1 and 2). We will test the validity of the new random features Algorithm~\ref{alg:HE} on two available datasets, MNIST and CIFAR10. In addition, we construct a new dataset called VGG11, which takes as input the last convolutional layer of the architecture VGG11 (the $8$-th layer) trained on CIFAR10, and as outputs the CIFAR10 labels. 

We plot in Figure \ref{fig:condition} the condition number of the empirical kernel matrix as depth increases. Empirically, we see that depth improves the kernel matrix's condition number, as discussed in Section 4. Note that, unlike the other three activation functions, the Sigmoid activation's kernel has the slowest decay as depth increases. This behavior of the Sigmoid activation results from the fact that the activation has a significant component when projected on the first $\iota$ Hermite polynomials (see Figure \ref{fig:coeff}), which further leads to a smaller implicit regularization due to truncation at level $\iota$. In contrast, ReLU activation has a much smaller component for the first $\iota$ Hermite coefficients, and therefore the condition number decays much faster. With the Sigmoid activation function, the compositional kernel can tolerate much higher depths such that the lower order Hermite coefficients do not vanish.

For each dataset, we run multi-class logistic regression (a simple one-layer neural network) with $10$ categories, using the random features generated by Algorithm~\ref{alg:HE}, with truncation level $\iota=20$. We consider four activations and three compositional depths, in total $12$ experiments. We run vanilla stochastic gradient descent as the training algorithm with batches of size $256$. The results are plotted in Figure \ref{fig:datasets} and the numerical results are displayed in Table \ref{fig:datasets}. We remark that only for the Sigmoid activation, the compositional kernels help improve the test accuracy. We postulate that this is because of the slower decay of the lower order Hermite coefficients of the compositional kernel, allowing one to vary the depth $L$ for a favorable trade-off between memorization and generalization.

\begin{figure}[!htbp]
  \begin{center}
  \includegraphics[scale=0.6]{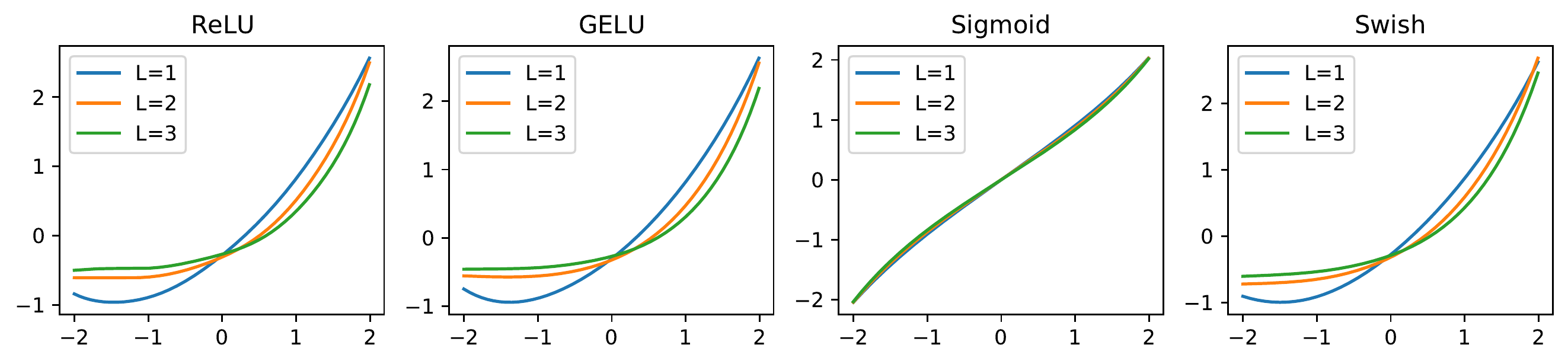}
  \end{center}
  \caption{Compressed activation functions truncated at level $\iota=20$ as in Algorithm~\ref{alg:HE}.}
  \label{fig:compressed-act}
\end{figure}

\begin{figure}[!htbp]
  \begin{center}
  \includegraphics[scale=0.5]{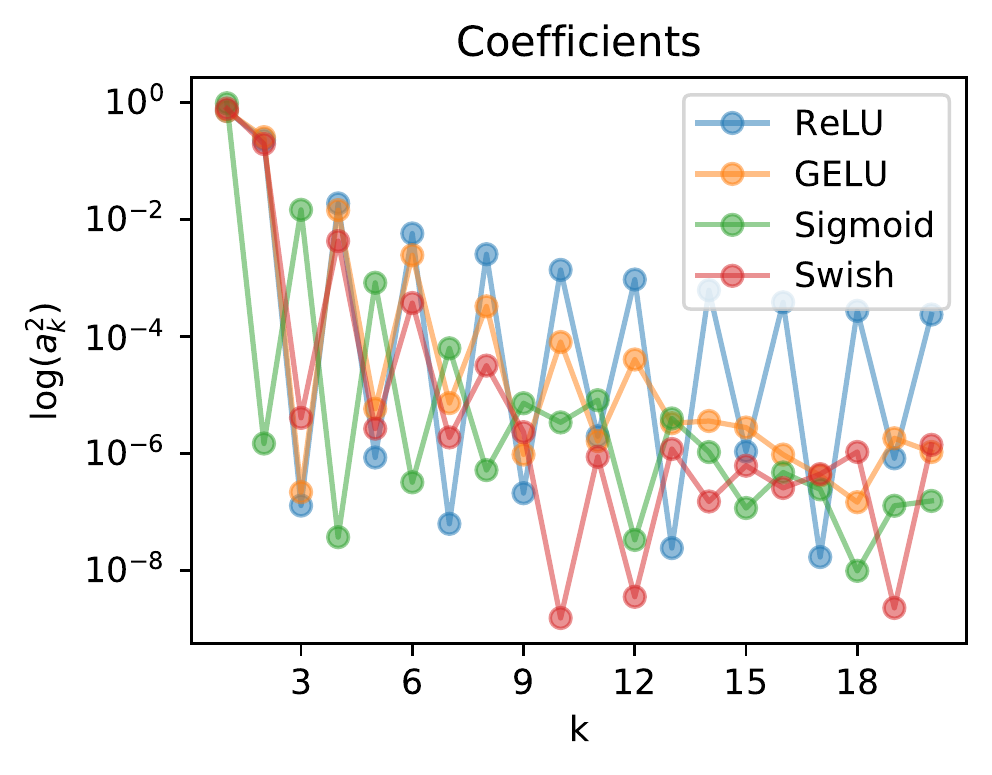}
  \includegraphics[scale=0.5]{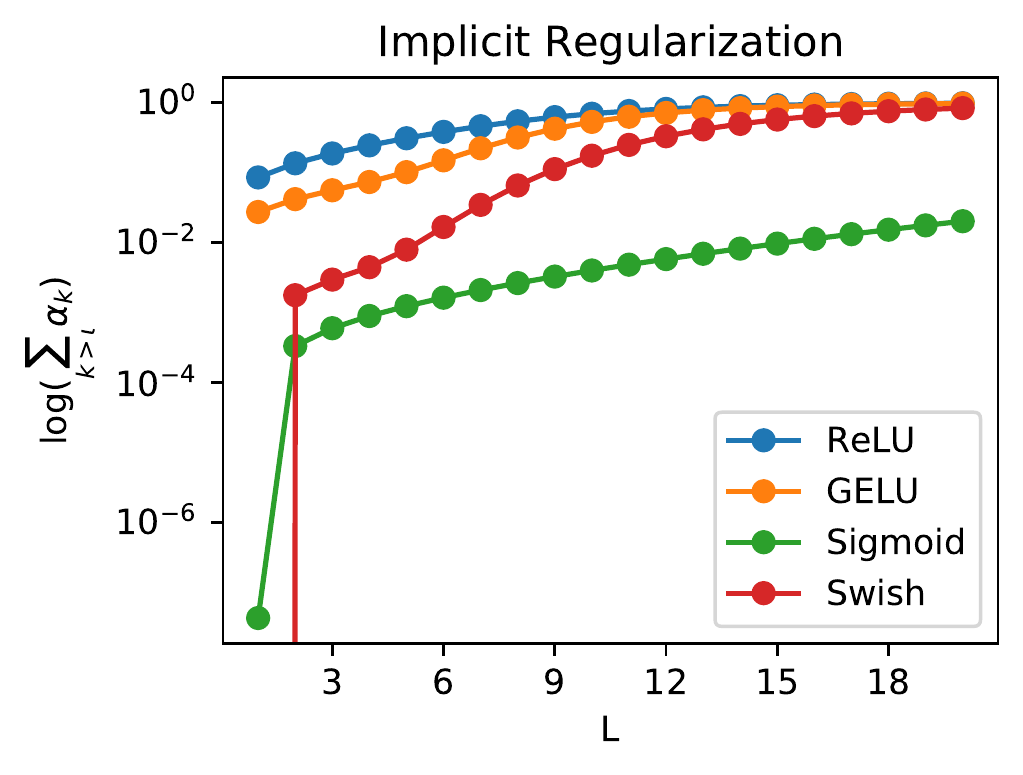}
  \end{center}
  \caption{On the left: estimated coefficients of the activation functions as in Algorithm~\ref{alg:coef_HE}. On the right: the amount of implicit regularization $\sum_{k>\iota}\alpha_{k}$ of the compressed activation function due to the truncation level $\iota=20$, as defined in Theorem \ref{thm:ekm-truncation-regularization}.}
  \label{fig:coeff}
\end{figure}

\begin{figure}[!htbp]

  \begin{center}

  \includegraphics[scale=0.5]{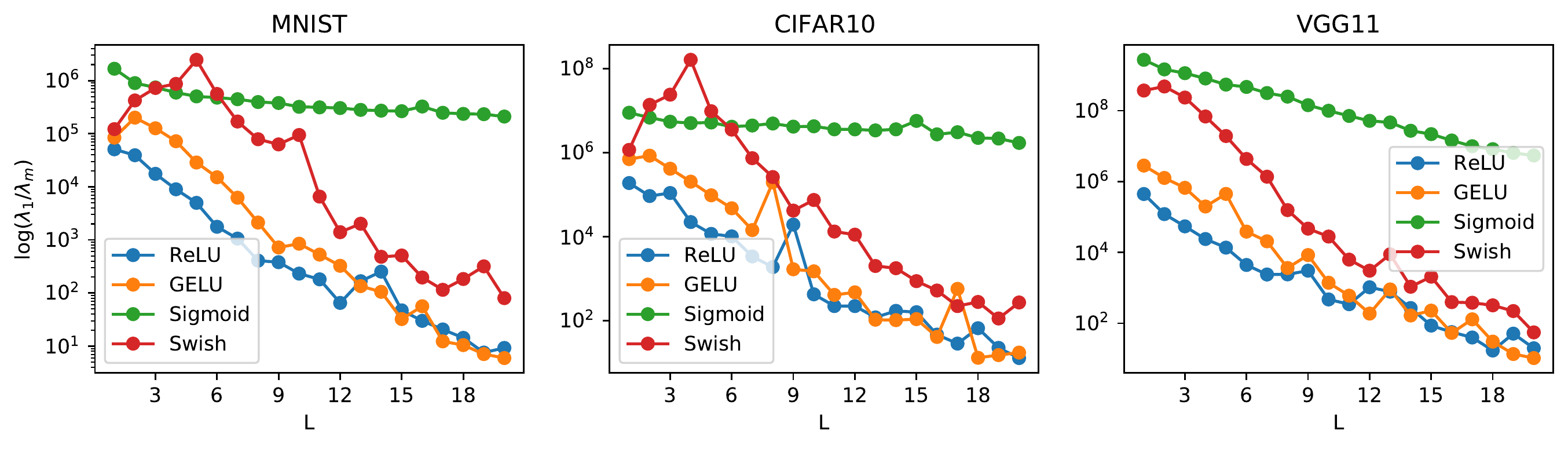}

  \end{center}
  \caption{Condition number of the empirical kernel matrix $\lambda_{1}(\Psi\Psi^{T}/m)/\lambda_{m}(\Psi\Psi^{T}/m)$ as a function of depth, where the random features matrix $\Psi$ is defined in Algorithm~\ref{alg:HE}.}
    \label{fig:condition}

\end{figure}

\begin{figure}[!htbp]
  \begin{center}
  \includegraphics[scale=0.5]{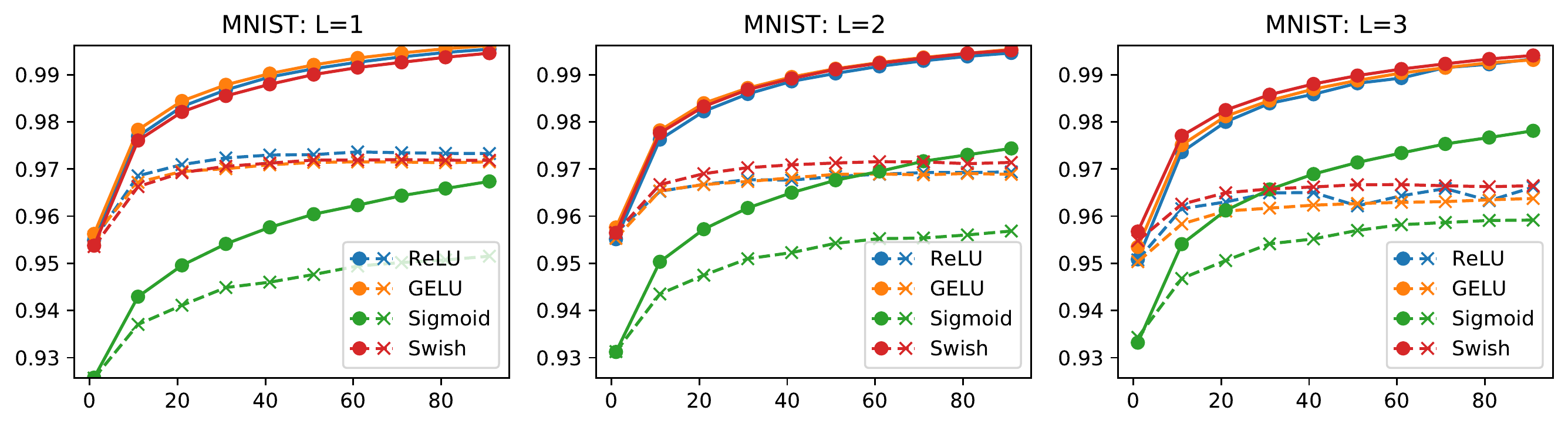}
  \includegraphics[scale=0.5]{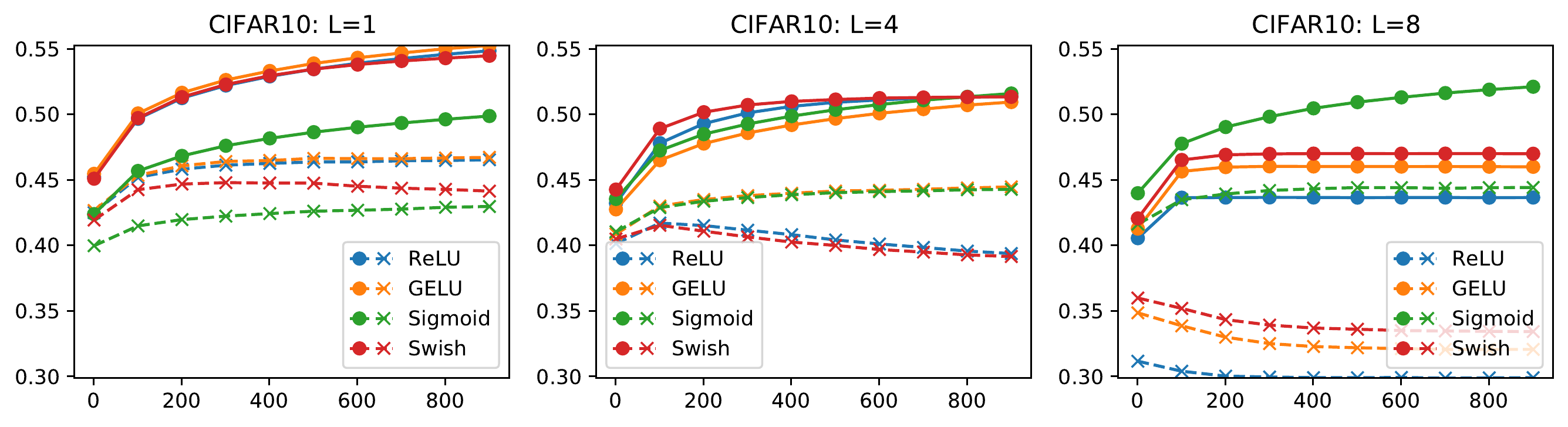}
  \includegraphics[scale=0.5]{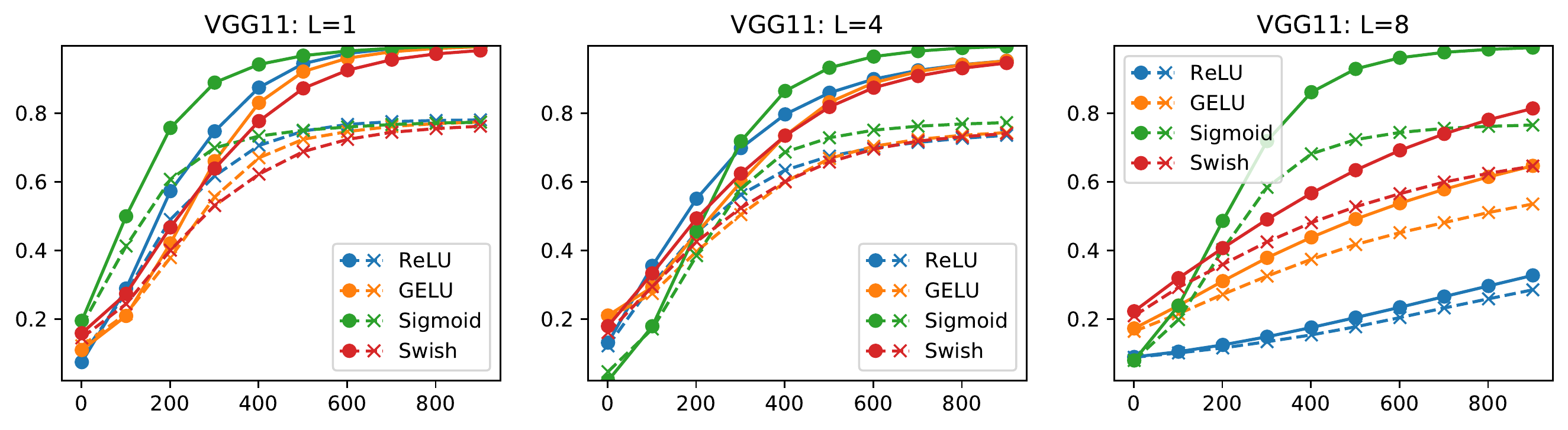}
  \end{center}
  \caption{Accuracy for varying activation functions and depths: Solid lines represent the training accuracy and dashed lines represent the testing accuracy. We used $100$ epochs ($\mathrm{lr}=10^{-2}$), $1000$ epochs ($\mathrm{lr}=10^{-3}$), and $1000$ epochs ($\mathrm{lr}=10^{-7}$) for MNIST, CIFAR10, and VGG11, respectively. The number of random features in every experiment was set to $m=2048$ (yielding $10\times 2048$ trainable parameters).}
  \label{fig:datasets}
\end{figure}

\begin{table}[!htbp]
  \begin{center}
    \begin{tabular}{c |c|c c c|c c c|c c c|c c c}
    \toprule
    \multicolumn{2}{c|}{Activation}  &\multicolumn{3}{c|}{ReLU} & \multicolumn{3}{c|}{GeLU}&\multicolumn{3}{c|}{Sigmoid}&\multicolumn{3}{c}{Swish}\\

    \hline
    \multicolumn{2}{c|}{Depth}  &$1$ & $2$ & $3$ & $1$ & $2$ & $3$ & $1$ & $2$ & $3$ & $1$ & $2$ & $3$\\
    \midrule
    \multirow{2}{*}{MNIST} & Train & $1.00$ & $0.99$ & $0.99$ & $1.00$ & $1.00$ & $0.99$ & $0.97$ & $0.97$ & $0.98$ & $0.99$ & $1.00$ & $0.99$ \\
    {} & Test & $0.97$ & $0.97$ & $0.97$ & $0.97$ & $0.97$ & $0.96$ & $0.95$ & $0.96$ & $0.96$ & $0.97$ & $0.97$ & $0.97$ \\
    
    \hline
    \multicolumn{2}{c|}{Depth}  &$1$ & $4$ & $8$ & $1$ & $4$ & $8$ & $1$ & $4$ & $8$ & $1$ & $4$ & $8$\\
    \midrule
    \multirow{2}{*}{CIFAR10}& Train & $0.55$ & $0.51$ & $0.44$ & $0.55$ & $0.51$ & $0.46$ & $0.50$ & $0.52$ & $0.52$ & $0.55$ & $0.51$ & $0.47$ \\
    {} & Test & $0.47$ & $0.39$ & $0.30$ & $0.47$ & $0.44$ & $0.32$ & $0.43$ & $0.44$ & $0.44$ & $0.44$ & $0.39$ & $0.33$ \\ 

    \hline
    \multicolumn{2}{c|}{Depth}  &$1$ & $4$ & $8$ & $1$ & $4$ & $8$ & $1$ & $4$ & $8$ & $1$ & $4$ & $8$\\
    \midrule
    \multirow{2}{*}{VGG11}& Train & $1.00$ & $0.96$ & $0.34$ & $0.99$ & $0.96$ & $0.66$ & $1.00$ & $1.00$ & $0.99$ & $0.99$ & $0.95$ & $0.83$ \\
    {} & Test & $0.78$ & $0.74$ & $0.30$ & $0.77$ & $0.75$ & $0.54$ & $0.77$ & $0.77$ & $0.77$ & $0.76$ & $0.74$ & $0.65$ \\ 

    \bottomrule
\end{tabular}
  \end{center}
  \caption{Accuracy percentage based on activation and depths. }
  \label{tbl:dataset}
\end{table}

\newpage
\bibliography{mehler-composition-kernel}

\begin{thebibliography}{41}
\providecommand{\natexlab}[1]{#1}
\providecommand{\url}[1]{\texttt{#1}}
\expandafter\ifx\csname urlstyle\endcsname\relax
  \providecommand{\doi}[1]{doi: #1}\else
  \providecommand{\doi}{doi: \begingroup \urlstyle{rm}\Url}\fi

\bibitem[Atkinson and Han(2012)]{AH12}
Kendall~E. Atkinson and Weimin Han.
\newblock \emph{Spherical Harmonics and Approximations on the Unit Sphere: An
  Introduction}.
\newblock Number 2044 in Lecture Notes in Mathematics. {Springer}, {Heidelberg
  Dordrecht London New York}, 2012.
\newblock ISBN 978-3-642-25982-1 978-3-642-25983-8.
\newblock OCLC: 781826006.

\bibitem[Bach(2016)]{Bac16}
Francis Bach.
\newblock Breaking the {{Curse}} of {{Dimensionality}} with {{Convex Neural
  Networks}}.
\newblock \emph{arXiv:1412.8690 [cs, math, stat]}, October 2016.

\bibitem[Bartlett et~al.(2019)Bartlett, Long, Lugosi, and Tsigler]{BLLT19}
Peter~L. Bartlett, Philip~M. Long, G{\'a}bor Lugosi, and Alexander Tsigler.
\newblock Benign {{Overfitting}} in {{Linear Regression}}.
\newblock June 2019.

\bibitem[Belkin et~al.(2018{\natexlab{a}})Belkin, Hsu, Ma, and Mandal]{BHMM18}
Mikhail Belkin, Daniel Hsu, Siyuan Ma, and Soumik Mandal.
\newblock Reconciling modern machine learning and the bias-variance trade-off.
\newblock \emph{arXiv:1812.11118 [cs, stat]}, December 2018{\natexlab{a}}.

\bibitem[Belkin et~al.(2018{\natexlab{b}})Belkin, Ma, and Mandal]{BMM18}
Mikhail Belkin, Siyuan Ma, and Soumik Mandal.
\newblock To understand deep learning we need to understand kernel learning.
\newblock February 2018{\natexlab{b}}.

\bibitem[Belkin et~al.(2018{\natexlab{c}})Belkin, Rakhlin, and Tsybakov]{BRT18}
Mikhail Belkin, Alexander Rakhlin, and Alexandre~B. Tsybakov.
\newblock Does data interpolation contradict statistical optimality?
\newblock June 2018{\natexlab{c}}.

\bibitem[Belkin et~al.(2019)Belkin, Hsu, and Xu]{BHX19}
Mikhail Belkin, Daniel Hsu, and Ji~Xu.
\newblock Two models of double descent for weak features.
\newblock \emph{arXiv:1903.07571 [cs, stat]}, March 2019.

\bibitem[Caponnetto and Vito(2006)]{CV06}
A~Caponnetto and E~De Vito.
\newblock {{Optimal Rates for Regularized Least}}-{{Squares Algorithm}}.
\newblock page~32, 2006.

\bibitem[Chizat and Bach(2018{\natexlab{a}})]{CB18}
Lenaic Chizat and Francis Bach.
\newblock A {{Note}} on {{Lazy Training}} in {{Supervised Differentiable
  Programming}}.
\newblock \emph{arXiv:1812.07956 [cs, math]}, December 2018{\natexlab{a}}.

\bibitem[Chizat and Bach(2018{\natexlab{b}})]{CB18a}
Lenaic Chizat and Francis Bach.
\newblock On the {{Global Convergence}} of {{Gradient Descent}} for
  {{Over}}-parameterized {{Models}} using {{Optimal Transport}}.
\newblock May 2018{\natexlab{b}}.

\bibitem[Cho and Saul(2009)]{CS09}
Youngmin Cho and Lawrence~K. Saul.
\newblock Kernel {{Methods}} for {{Deep Learning}}.
\newblock In Y.~Bengio, D.~Schuurmans, J.~D. Lafferty, C.~K.~I. Williams, and
  A.~Culotta, editors, \emph{Advances in {{Neural Information Processing
  Systems}} 22}, pages 342--350. {Curran Associates, Inc.}, 2009.

\bibitem[Daniely et~al.(2017{\natexlab{a}})Daniely, Frostig, Gupta, and
  Singer]{DFGS17}
Amit Daniely, Roy Frostig, Vineet Gupta, and Yoram Singer.
\newblock Random {{Features}} for {{Compositional Kernels}}.
\newblock \emph{arXiv:1703.07872 [cs]}, March 2017{\natexlab{a}}.

\bibitem[Daniely et~al.(2017{\natexlab{b}})Daniely, Frostig, and Singer]{DFS17}
Amit Daniely, Roy Frostig, and Yoram Singer.
\newblock Toward {{Deeper Understanding}} of {{Neural Networks}}: {{The Power}}
  of {{Initialization}} and a {{Dual View}} on {{Expressivity}}.
\newblock \emph{arXiv:1602.05897 [cs, stat]}, May 2017{\natexlab{b}}.

\bibitem[Dou and Liang(2020)]{DL19}
Xialiang Dou and Tengyuan Liang.
\newblock Training neural networks as learning data-adaptive kernels:
  {{Provable}} representation and approximation benefits.
\newblock \emph{Journal of the American Statistical Association}, 0\penalty0
  (0):\penalty0 1--14, 2020.
\newblock \doi{10.1080/01621459.2020.1745812}.

\bibitem[Du et~al.(2018)Du, Zhai, Poczos, and Singh]{DZPS18}
Simon~S. Du, Xiyu Zhai, Barnabas Poczos, and Aarti Singh.
\newblock Gradient {{Descent Provably Optimizes Over}}-parameterized {{Neural
  Networks}}.
\newblock October 2018.

\bibitem[Feldman(2019)]{feldman2019does}
Vitaly Feldman.
\newblock Does learning require memorization? a short tale about a long tail.
\newblock \emph{arXiv preprint arXiv:1906.05271}, 2019.

\bibitem[Hastie et~al.(2019)Hastie, Montanari, Rosset, and Tibshirani]{HMRT19}
Trevor Hastie, Andrea Montanari, Saharon Rosset, and Ryan~J. Tibshirani.
\newblock Surprises in {{High}}-{{Dimensional Ridgeless Least Squares
  Interpolation}}.
\newblock March 2019.

\bibitem[Jacot et~al.(2019)Jacot, Gabriel, and Hongler]{JGH19}
Arthur Jacot, Franck Gabriel, and Cl{\'e}ment Hongler.
\newblock Freeze and {{Chaos}} for {{DNNs}}: An {{NTK}} view of {{Batch
  Normalization}}, {{Checkerboard}} and {{Boundary Effects}}.
\newblock \emph{arXiv:1907.05715 [cs, stat]}, July 2019.

\bibitem[Kar and Karnick(2012)]{KK12}
Purushottam Kar and Harish Karnick.
\newblock Random {{Feature Maps}} for {{Dot Product Kernels}}.
\newblock \emph{arXiv:1201.6530 [cs, math, stat]}, March 2012.

\bibitem[Kesten and Stigum(1966)]{KS66}
H.~Kesten and B.~P. Stigum.
\newblock A {{Limit Theorem}} for {{Multidimensional Galton}}-{{Watson
  Processes}}.
\newblock \emph{The Annals of Mathematical Statistics}, 37\penalty0
  (5):\penalty0 1211--1223, October 1966.
\newblock ISSN 0003-4851, 2168-8990.
\newblock \doi{10.1214/aoms/1177699266}.

\bibitem[Liang and Rakhlin(2020)]{LR18}
Tengyuan Liang and Alexander Rakhlin.
\newblock Just interpolate: {{Kernel}} ``{{Ridgeless}}'' regression can
  generalize.
\newblock \emph{The Annals of Statistics}, 48\penalty0 (3):\penalty0
  1329--1347, June 2020.
\newblock \doi{10.1214/19-AOS1849}.

\bibitem[Liang et~al.(2020)Liang, Rakhlin, and Zhai]{LRZ20}
Tengyuan Liang, Alexander Rakhlin, and Xiyu Zhai.
\newblock On the multiple descent of minimum-norm interpolants and restricted
  lower isometry of kernels.
\newblock In Jacob Abernethy and Shivani Agarwal, editors, \emph{Proceedings of
  33rd Conference on Learning Theory}, volume 125 of \emph{Proceedings of
  Machine Learning Research}, pages 2683--2711. {PMLR}, July 2020.

\bibitem[Lyons and Peres(2016)]{LP16}
Russell Lyons and Yuval Peres.
\newblock \emph{Probability on {{Trees}} and {{Networks}}}.
\newblock {Cambridge University Press}, {Cambridge}, 2016.
\newblock ISBN 978-1-316-67281-5.
\newblock \doi{10.1017/9781316672815}.

\bibitem[Mei and Montanari(2019)]{MM19}
Song Mei and Andrea Montanari.
\newblock The generalization error of random features regression: {{Precise}}
  asymptotics and double descent curve.
\newblock \emph{arXiv:1908.05355 [math, stat]}, October 2019.

\bibitem[Mei et~al.(2018)Mei, Montanari, and Nguyen]{MMN18}
Song Mei, Andrea Montanari, and Phan-Minh Nguyen.
\newblock A {{Mean Field View}} of the {{Landscape}} of {{Two}}-{{Layers Neural
  Networks}}.
\newblock \emph{arXiv:1804.06561 [cond-mat, stat]}, August 2018.

\bibitem[Nakkiran et~al.(2020)Nakkiran, Venkat, Kakade, and
  Ma]{nakkiran2020optimal}
Preetum Nakkiran, Prayaag Venkat, Sham Kakade, and Tengyu Ma.
\newblock Optimal regularization can mitigate double descent.
\newblock \emph{arXiv preprint arXiv:2003.01897}, 2020.

\bibitem[Neal(1996{\natexlab{a}})]{Nea96}
Radford~M. Neal.
\newblock \emph{Bayesian {{Learning}} for {{Neural Networks}}}, volume 118 of
  \emph{Lecture {{Notes}} in {{Statistics}}}.
\newblock {Springer New York}, {New York, NY}, 1996{\natexlab{a}}.
\newblock ISBN 978-0-387-94724-2 978-1-4612-0745-0.
\newblock \doi{10.1007/978-1-4612-0745-0}.

\bibitem[Neal(1996{\natexlab{b}})]{Nea96a}
Radford~M. Neal.
\newblock Priors for {{Infinite Networks}}.
\newblock In Radford~M. Neal, editor, \emph{Bayesian {{Learning}} for {{Neural
  Networks}}}, Lecture {{Notes}} in {{Statistics}}, pages 29--53. {Springer},
  {New York, NY}, 1996{\natexlab{b}}.
\newblock ISBN 978-1-4612-0745-0.
\newblock \doi{10.1007/978-1-4612-0745-0_2}.

\bibitem[Nguyen and Pham(2020)]{nguyen2020rigorous}
Phan-Minh Nguyen and Huy~Tuan Pham.
\newblock A rigorous framework for the mean field limit of multilayer neural
  networks.
\newblock \emph{arXiv preprint arXiv:2001.11443}, 2020.

\bibitem[Pennington et~al.(2015)Pennington, Yu, and Kumar]{PYK15}
Jeffrey Pennington, Felix Xinnan~X Yu, and Sanjiv Kumar.
\newblock Spherical {{Random Features}} for {{Polynomial Kernels}}.
\newblock In C.~Cortes, N.~D. Lawrence, D.~D. Lee, M.~Sugiyama, and R.~Garnett,
  editors, \emph{Advances in {{Neural Information Processing Systems}} 28},
  pages 1846--1854. {Curran Associates, Inc.}, 2015.

\bibitem[Poole et~al.(2016)Poole, Lahiri, Raghu, Sohl-Dickstein, and
  Ganguli]{NIPS2016_6322}
Ben Poole, Subhaneil Lahiri, Maithra Raghu, Jascha Sohl-Dickstein, and Surya
  Ganguli.
\newblock Exponential expressivity in deep neural networks through transient
  chaos.
\newblock In D.~D. Lee, M.~Sugiyama, U.~V. Luxburg, I.~Guyon, and R.~Garnett,
  editors, \emph{Advances in Neural Information Processing Systems 29}, pages
  3360--3368. Curran Associates, Inc., 2016.

\bibitem[Rahimi and Recht(2008)]{RR08}
Ali Rahimi and Benjamin Recht.
\newblock Random {{Features}} for {{Large}}-{{Scale Kernel Machines}}.
\newblock In J.~C. Platt, D.~Koller, Y.~Singer, and S.~T. Roweis, editors,
  \emph{Advances in {{Neural Information Processing Systems}} 20}, pages
  1177--1184. {Curran Associates, Inc.}, 2008.

\bibitem[Rahimi and Recht(2009)]{RR09}
Ali Rahimi and Benjamin Recht.
\newblock Weighted {{Sums}} of {{Random Kitchen Sinks}}: {{Replacing}}
  minimization with randomization in learning.
\newblock In D.~Koller, D.~Schuurmans, Y.~Bengio, and L.~Bottou, editors,
  \emph{Advances in {{Neural Information Processing Systems}} 21}, pages
  1313--1320. {Curran Associates, Inc.}, 2009.

\bibitem[Rotskoff and {Vanden-Eijnden}(2018)]{RV18}
Grant~M. Rotskoff and Eric {Vanden-Eijnden}.
\newblock Trainability and {{Accuracy}} of {{Neural Networks}}: {{An
  Interacting Particle System Approach}}.
\newblock May 2018.

\bibitem[Schoenberg(1942)]{Sch42}
I.~J. Schoenberg.
\newblock Positive definite functions on spheres.
\newblock \emph{Duke Mathematical Journal}, 9\penalty0 (1):\penalty0 96--108,
  March 1942.
\newblock ISSN 0012-7094.
\newblock \doi{10.1215/S0012-7094-42-00908-6}.

\bibitem[Shankar et~al.(2020)Shankar, Fang, Guo, Fridovich-Keil, Schmidt,
  Ragan-Kelley, and Recht]{shankar2020neural}
Vaishaal Shankar, Alex Fang, Wenshuo Guo, Sara Fridovich-Keil, Ludwig Schmidt,
  Jonathan Ragan-Kelley, and Benjamin Recht.
\newblock Neural kernels without tangents.
\newblock 2020.

\bibitem[Sirignano and Spiliopoulos(2018)]{SS18}
Justin Sirignano and Konstantinos Spiliopoulos.
\newblock Mean {{Field Analysis}} of {{Neural Networks}}: {{A Law}} of {{Large
  Numbers}}.
\newblock May 2018.

\bibitem[Stitson et~al.(1999)Stitson, Gammerman, Vapnik, Vovk, Watkins, and
  Weston]{stitson1997anova}
Mark Stitson, Alex Gammerman, Vladimir Vapnik, Volodya Vovk, Christopher
  Watkins, and Jason Weston.
\newblock Support vector regression with anova decomposition kernels.
\newblock 04 1999.

\bibitem[Woodworth et~al.(2019)Woodworth, Gunasekar, Savarese, Moroshko, Golan,
  Lee, Soudry, and Srebro]{WGS+19}
Blake Woodworth, Suriya Gunasekar, Pedro Savarese, Edward Moroshko, Itay Golan,
  Jason Lee, Daniel Soudry, and Nathan Srebro.
\newblock Kernel and {{Rich Regimes}} in {{Overparametrized Models}}.
\newblock \emph{arXiv:1906.05827 [cs, stat]}, September 2019.

\bibitem[Yang(2019)]{yang2019scaling}
Greg Yang.
\newblock Scaling limits of wide neural networks with weight sharing: Gaussian
  process behavior, gradient independence, and neural tangent kernel
  derivation.
\newblock \emph{arXiv preprint arXiv:1902.04760}, 2019.

\bibitem[Zhang et~al.(2017)Zhang, Bengio, Hardt, Recht, and Vinyals]{ZBH+17}
Chiyuan Zhang, Samy Bengio, Moritz Hardt, Benjamin Recht, and Oriol Vinyals.
\newblock Understanding deep learning requires rethinking generalization.
\newblock \emph{arXiv:1611.03530 [cs]}, February 2017.

\end{thebibliography}

\appendix

\section{Proofs}

\subsection{Proofs in Section~\ref{sec:compo-kernel-branching-proc}}

\begin{proof}[Proof of Lemma~\ref{lem:key-duality}]
  By properties of the normal distribution, we have
\begin{align}
  \EE_{\t \sim \cN(0, \I_d)} \left[ \sigma\left( \t^\top x/\|x\| \right) \sigma\left( \t^\top z/\|z\| \right) \right] = \EE_{(\tilde x,\tilde z)\sim \cN_{\rho}}\left[\sigma\left(\tilde{x}\right)\sigma\left(\tilde{z}\right)\right],
\end{align}
where $\cN_{\rho}$ denotes the standard bivariate normal distribution with correlation $\rho = \langle x/\| x\|, z/\| z \| \rangle$. Expand the above expression explicitly and use Mehler's formula in Proposition~\ref{prop:mehler} to obtain
\begin{align}
\EE_{(\tilde{x},\tilde{z})\sim \cN_{\rho}}\left[\sigma\left(\tilde{x}\right)\sigma\left(\tilde{z}\right)\right]
&=\int_{\mathbb{R}}\int_{\mathbb{R}}\sigma(\tilde{x})\sigma(\tilde{z})\frac{1}{2\pi\sqrt{1-\rho^2}}\exp\left\{-\frac{\tilde{x}^2+\tilde{z}^2-2\rho\tilde{x}\tilde{z}}{2(1-\rho^2)}\right\}d\tilde{x}d\tilde{z}\\
&=\int_{\mathbb{R}}\int_{\mathbb{R}}\sigma(\tilde{x})\sigma(\tilde{z})\frac{1}{2\pi}\exp\left\{-\frac{\tilde{x}^2+\tilde{z}^2}{2}\right\}\sum_{k\geq0}\rho^{k}h_{k}(\tilde{x})h_{k}(\tilde{z})\\
&=\sum_{k\geq 0}\rho^{k}\EE_{\tilde{x}\sim \cN(0,1)}[\sigma(\tilde{x})h_{k}(\tilde{x})]\EE_{\tilde{z}\sim \cN(0,1)}[\sigma(\tilde{z})h_{k}(\tilde{z})]\\
&=\sum_{k\geq 0}a_{k}^2\rho^{k} =G_{\sigma}(\rho) \enspace.
\end{align}

\end{proof}

\begin{proof}[Proof of Lemma~\ref{lem:composition-kernel}]
We can prove by induction on $L$. For $L=0$, it follows from the definition that
\begin{align}
K_{\sigma}^{(0)}(\rho)=\rho=\EE[\rho^{Z_{\sigma,0}}].
\end{align}
For $L\geq 1$, we have by induction
\begin{align*}
K_{\sigma}^{(L)}(\rho)
&=K_{\sigma}^{(L-1)}(K_{\sigma}(\rho))=K_{\sigma}^{(L-1)}(G(\rho))\\
&=\EE[G(\rho)^{Z_{\sigma,L-1}}]=\EE[\EE[\rho^{Y_{\sigma}}]^{Z_{\sigma,L-1}}]\\
&=\EE[\rho^{\sum_{i=1}^{Z_{\sigma,L-1}}Y_{\sigma}^{(L-1)}}] =\EE[\rho^{Z_{\sigma,L}}] \enspace.
\end{align*}
\end{proof}

We will use the following facts about branching process and composition of PGFs.
\begin{proposition}[Composition of PGFs]
  The PGF associated with the size of the $L$-th generation $Z_L$ in a branching process satisfies
  \begin{align}
    \EE[s^{Z_{L}}] = \underbrace{ G \circ  \cdots \circ G }_{\text{composite $L$ times}}(s) =: G^{(L)}(s)
  \end{align}
  where $G(s)$ is the PGF of $Y$ in Definition~\ref{def:galton-watson}.
\end{proposition}

\begin{theorem}[\cite{KS66}, See also Theorem 12.3 in \cite{LP16}]
  \label{thm:kensten-stigum}
  Consider the setting in Theorem \ref{thm:rescaled-limit}. On the one hand if $\mu_\star<\infty$
  then there exist a random variable $W_{\sigma} >0$
  \begin{align}
    \lim_{L\rightarrow \infty} Z_{\sigma,L}/\mu^{L}  = W_{\sigma}, ~~\text{a.s. conditional on non-extinction}\enspace.
  \end{align}
  Here $\EE [W_{\sigma}] = 1$, and $\Var[W_{\sigma}] = \frac{\Var[Y_\sigma]}{\EE[Y_\sigma] ( \EE[Y_\sigma]-1 )}$.
  On the other hand if $\mu_\star=\infty$ then for any constant $m>0$, we have
  \begin{align}
    \label{eq:rescaled}
    \lim_{L\rightarrow \infty} Z_{\sigma,L}/m^{L} = \infty, ~~\text{a.s. conditional on non-extinction}\enspace.
  \end{align}
\end{theorem}

\begin{proposition} [Properties of PGF]
  \label{prop:prop-pgf}
  Consider the function $G_{\sigma}(\cdot)$ defined in Lemma \ref{lem:key-duality}. Let $\xi$ be smallest fixed point with  $G_{\sigma}(\xi)=\xi$, in $[0,1]$. If $a_{1}^2=1$, then $G_{\sigma}(\cdot)$ is linear with $G_{\sigma}(\rho)=\rho$. For $a_{1}\neq 1$, we have 
  \begin{enumerate}
    \item (Extinction). $\xi=1$ if and only if $G'(1)\leq 1$.
    \item (Fixed Points). $G_{\sigma}(\cdot)$ has only two fixed points $\xi$ and $1$ on $[0,1)$.
    \item (Compositional limit). $\lim_{L\to\infty}G^{(L)}_{\sigma}(\rho)=\xi$ for any $\rho\in[0,1)$. 
  
  \end{enumerate}
\end{proposition}
\begin{proof} Note that $G_{\sigma}(\cdot)$ has a fixed point at $1$ and it is continuous. Therefore, $\xi$ is well defined. The first two properties \textit{extinction} and \textit{fixed points} are just a restatement of Proposition \ref{prop:extinction-criterion}. Let's prove the third property \textit{compositional limit}. There are two cases $\xi=1$ and $\xi <1$.

If $\xi = 1$, then $G_{\sigma}'(1)\leq 1$. Hence, by convexity, we have $G_{\sigma}(\rho)\geq \rho\geq 0$ for any $\rho \in [0,1]$. Therefore, $\{G_{\sigma}^{(L)}(\rho)\colon L\in\mathbb{Z}^{\geq 0}\}$ is a non-decreasing sequence, for any $\rho\in[0,1]$, and it's limit converges to the unique fixed point $\xi=1$ of $G_{\sigma}(\cdot)$ on $[0,1]$. 

If $\xi <1$, then $G_{\sigma}'(1)=\mu>1$. First, let's look at $\rho\in[\xi,1)$. By convexity, we have $G_{\sigma}(\rho)\leq \rho$ for any  $\rho\in [\xi,1]$. Moreover, $G_{\sigma}(\cdot)$ is non-decreasing on $[\xi,1)$. Therefore, we have $\xi=G_{\sigma}(\xi)\leq G_{\sigma}(\rho)\leq \rho$ for any $\rho\in[\xi,1)$. Then, we have that $\{G_{\sigma}^{(L)}(\rho)\colon L\in\mathbb{Z}^{\geq 0}\}$ is a non-increasing sequence, for any $\rho\in[\xi,1)$, and it's limit converges to the unique fixed point $\xi$ of $G_{\sigma}$ on $[\xi,1)$.

Now, let's look at $\rho\in [0,\xi]$. By continuity of $G_{\sigma}(\cdot)$ and property \textit{fixed points}, we have that $G_{\sigma}(\rho)>\rho$ for any $\rho \in [0,\xi)$. If for any $L$ we have that $G_{\sigma}^{(L)}\in[\xi,1)$, then the result follows as a consequence of the case $\rho\in[\xi,1)$ If  $\xi>G_{\sigma}^{(L)}(\rho)>\rho$ for all $L\in\mathbb{Z}^{\geq 0}$, then the result follows by the same monotonicity argument as in the case $\rho\in[\xi,1)$.
\end{proof}

\begin{remark} 
\label{rmk:add-assump}
Under the same notation as Proposition \ref{prop:prop-pgf},  we will state three sufficient conditions for which the \textit{compositional limit} property of $G_{\sigma}$ extends for $\rho\in(-1,1)$:
\begin{enumerate}
  \item (Reduction to positive side). $G_{\sigma}(\cdot)$ is non-negative; for example, an even function $G_{\sigma}(\cdot)$. In this case, after the first iteration we will be in the realm of $[0,1]$, that is $G_{\sigma}(\rho)\in[0,1]$, and the result follows by the \textit{compositional limit} property applied to $G_{\sigma}(\rho)$. 
  \item (Extension of fixed points). $G_{\sigma}(\cdot)$ has only two fixed points $\xi$ and $1$; for example, a convex function $G_{\sigma}(\cdot)$. In this case, the result follows because the proof of  \textit{compositional limit} property for $\rho\in[0,\xi]$ relies solely on the \textit{fixed points} property. Therefore, we can extend the part in which  $\rho\in(0,\xi]$ to $\rho\in(-1,\xi)$.
  \item (Centered Activation) $G_{\sigma}(0)=0$; for example, an odd function $G_{\sigma}(\cdot)$. In this case, the result follows because $\xi=0$ and the contraction behavior of $G_{\sigma}(\cdot)$
  \begin{align}
  |G_{\sigma}^{(L)}(\rho)|\leq G_{\sigma}^{(L)}(|\rho|)\leq G_{\sigma}^{(L-1)}(|\rho|)\cdot |\rho|\leq \ldots\leq|\rho|^{L}.
  \end{align}
\end{enumerate}
\end{remark}

\begin{proof}[Proof of Theorem~\ref{thm:rescaled-limit}] Recall that $K_{\sigma}^{(L)}=G^{(L)}_{\sigma}$.
\begin{itemize}
\item[(i)] Since $e^{-t}\in (0,1]$, the conclusion follows by applying Proposition \ref{prop:prop-pgf}'s \textit{compositional limit} and \textit{extinction} properties.
\item[(ii)] By applying  Lemma \ref{lem:composition-kernel} and Kesten-Stigum Theorem \ref{thm:kensten-stigum}, we have conditional on non-extinction that 
\begin{align}
\lim_{L\to\infty}K_{\sigma}^{(L)}(e^{-t/\mu^{L}})
&=\lim_{L\to\infty}\EE[e^{-t Z_{\sigma,L}/\mu^{L}}]\\
&=\EE[\lim_{L\to\infty}e^{-t Z_{\sigma,L}/\mu^{L}}]\\
&=\EE[e^{-t W_{\sigma}}],
\end{align}

where interchanging limit and expectation follows by the Dominated Convergence Theorem. The proof completes by recalling that the extinction probability is $\xi$.
\item[(iii)]Analogous as above, if $t\neq 0$, then for any $m>0$
\begin{align}
\lim_{L\to\infty}K_{\sigma}^{(L)}(e^{-t/m^{L}})
&=\lim_{L\to\infty}\EE[e^{-t Z_{\sigma,L}/m^{L}}]\\
&=\EE[\lim_{L\to\infty}e^{-t Z_{\sigma,L}/m^{L}}]\\
&=\EE[0].\\
&=0.
\end{align}
The case $t=0$ follows immediately from the fact that $G_{\sigma}(1)=1$.
\end{itemize}
\end{proof}

\begin{proof}[Proof of Corollary~\ref{cor:unscaled-limit}] The conclusions follow from a reformulation of Proposition \ref{prop:prop-pgf}'s \textit{compositional limit} and \textit{extinction} properties, with the addition of Remark \ref{rmk:add-assump}.

\end{proof} 

\subsection{Proofs in Section~\ref{sec:nonasymptotic-memorization}}

\subsubsection{Bounding the Path Depth}
\begin{definition} (Path depth) For any $\alpha\leq\beta\in(0,1)$, define the minimum depth of compositions of $G_{\sigma}$ to reach from $\beta$ to below $\alpha$ as $L_{\beta\to\alpha}$, namely
\begin{align}
L_{\beta\to\alpha} = \min\{L\geq 0\colon G_{\sigma}^{(L)}(\beta)\leq \alpha\}.
\end{align}
\end{definition}

\begin{lemma}[Path depth: upper bound] \label{lem:path-upp-bd} 
For any $\alpha\leq\beta\in(0,1)$, we have
\begin{align}
L_{\beta\to \alpha} \leq H_{\mathrm{upp}}(\alpha,\beta)\coloneqq\min\Bigg\{ \frac{\log\frac{\beta}{\alpha}}{\log \frac{\beta}{G_{\sigma}(\beta)}},\frac{\log\frac{1-\alpha}{1-\beta}}{\log \frac{1-G_{\sigma}(\alpha)}{1-\alpha}}\Bigg\} + 1.
\end{align}
\end{lemma}
\begin{proof}
\textbf{Right-side bound}. 
For all $t\in[\alpha,\beta]$, we have
\begin{align}
\frac{1-G_{\sigma}(t)}{1-t}=\sum_{k\geq 1}a_{k}^2(1+t+\ldots+t^{k-1})
\geq \sum_{k\geq 1}a_{k}^2(1+\alpha+\ldots+\alpha^{k-1})=\frac{1-G_{\sigma}(\alpha)}{1-\alpha} > 1.
\end{align}
Therefore, we have
\begin{align}
\alpha < G_{\sigma}^{(L_{\beta\to \alpha}-1)}(\beta)
&=1-(1-\beta)\prod_{\ell=0}^{L_{\beta\to \alpha}-2} \frac{1-G_{\sigma}^{(\ell+1)}(\beta)}{1-G_{\sigma}^{(\ell)}(\beta)}\\
&\leq 1- (1-\beta)\Big(\frac{1-G_{\sigma}(\alpha)}{1-\alpha}\Big)^{L_{\beta\to \alpha}-1}
\end{align}
Rearranging, we obtain
\begin{align}
\label{eq:top-bot-path}
L_{\beta\to \alpha}\leq \frac{\log\frac{1-\alpha}{1-\beta}}{\log\frac{1-G_{\sigma}(\alpha)}{1-\alpha}} + 1.
\end{align}

\textbf{Left-side bound}. For all $t\in[\alpha,\beta]$, we have
\begin{align}
\frac{G_{\sigma}(t)}{t}=\sum_{k\geq 1}a_{k}^2 t^{k-1}\leq \sum_{k\geq 1}a_{k}^2 \beta^{k-1}=\frac{G_{\sigma}(\beta)}{\beta} < 1
\end{align}
Thus, we have 
\begin{align}
 \alpha < G_{\sigma}^{(L_{\beta\to \alpha}-1)}(\beta)
&=\beta\cdot \prod_{\ell=0}^{L_{\beta\to \alpha}-2} \frac{G_{\sigma}^{(\ell+1)}(\beta)}{G_{\sigma}^{(\ell)}(\beta)}\\
&\leq \beta \cdot \Big(\frac{G_{\sigma}(\beta)}{\beta}\Big)^{L_{\beta\to \alpha}-1}.
\end{align}
Now, rearranging we have
\begin{align}
\label{eq:low-bot-path}
L_{\beta\to \alpha}\leq \frac{\log\frac{\beta}{\alpha}}{\log \frac{\beta}{G_{\sigma}(\beta)}} + 1.
\end{align}
\end{proof}

\begin{lemma} (Path depth: lower bound).\label{lem:path-low-bd} For any $\alpha\leq\beta\in(0,1)$, we have
\begin{align}
L_{\beta\to \alpha} \geq H_{\mathrm{low}}(\alpha,\beta)\coloneqq\max\Bigg\{ \frac{\log\frac{\beta}{\alpha}}{\log \frac{\alpha}{G_{\sigma}(\alpha)}},\frac{\log\frac{1-\alpha}{1-\beta}}{\log \frac{1-G_{\sigma}(\beta)}{1-\beta}}\Bigg\}
\end{align}
\end{lemma}
\begin{proof}
\textbf{Right-side bound}. For all $t\in[\alpha,\beta]$, we have
\begin{align}
1 < \frac{1-G_{\sigma}(t)}{1-t}=\sum_{k\geq 1}a_{k}^2(1+t+\ldots+t^{k-1})
\leq \sum_{k\geq 1}a_{k}^2(1+\beta+\ldots+\beta^{k-1})=\frac{1-G_{\sigma}(\beta)}{1-\beta}.
\end{align}
Therefore, we have
\begin{align}
\alpha&\geq G_{\sigma}^{(L_{\beta\to \alpha})}(\beta)\\
&= 1- (1-\beta)\prod_{\ell=0}^{L_{\beta\to \alpha}-1} \frac{1-G_{\sigma}^{(\ell+1)}(\beta)}{1-G_{\sigma}^{(\ell)}(\beta)}\\
&\geq 1-(1-\beta)\cdot \Big(\frac{1-G_{\sigma}(\beta)}{ 1-\beta}\Big)^{L_{\beta\to \alpha}}
\end{align}
Rearranging and recall the integer definition of $L_{\beta\to \alpha}$, we get that
\begin{align}
\label{eq:top-bot-path}
L_{\beta\to \alpha}\geq \big\lceil \frac{\log\frac{1-\alpha}{1-\beta}}{\log \frac{1-G_{\sigma}(\beta)}{1-\beta}} \big\rceil.
\end{align}

\textbf{Left-side bound}. For all $t\in[\alpha,\beta]$, we have
\begin{align}
1> \frac{G_{\sigma}(t)}{t}=\sum_{k\geq 1}a_{k}^2 t^{k-1}\geq \sum_{k\geq 1}a_{k}^2 \alpha^{k-1}=\frac{G_{\sigma}(\alpha)}{\alpha}
\end{align}
Thus, we have
\begin{align}
\alpha
&\geq G_{\sigma}^{(L_{\beta\to \alpha})}(\beta)\\
&=\beta\cdot \prod_{\ell=0}^{L_{\beta\to \alpha}-1} \frac{G_{\sigma}^{(\ell+1)}(\beta)}{G_{\sigma}^{(\ell)}(\beta)}\\
&\geq \beta \cdot \Big(\frac{G_{\sigma}(\alpha)}{\alpha}\Big)^{L_{\beta\to \alpha}}.
\end{align}
Now, rearranging we have
\begin{align}
\label{eq:low-bot-path}
L_{\beta\to \alpha}\geq \big\lceil \frac{\log\frac{\beta}{\alpha}}{\log \frac{\alpha}{G_{\sigma}(\alpha)}} \big\rceil.
\end{align}
\end{proof}

\begin{corollary} (Chain of Path Depth) \label{cor:path} For any $\alpha\leq\beta\in(0,1)$, we have
\begin{align}
\max_{s\in[\alpha,\beta]} \{ H_{\mathrm{low}}(\alpha,s) + H_{\mathrm{low}}(s,\beta) \} -1 \leq L_{\beta\to\alpha}\leq \min_{s\in[\alpha,\beta]} \{ H_{\mathrm{upp}}(\alpha,s) + H_{\mathrm{upp}}(s,\beta) \}.
\end{align}
\end{corollary}
\begin{proof} We have
\begin{align}
L_{\beta\to\alpha} = L_{\beta\to s}+  L_{G_{\sigma}^{(L_{\beta\to s})}(\beta)\to\alpha}.
\end{align}
By definition of $L_{\beta\to s}$, we have
\begin{align}
G_{\sigma}^{(L_{\beta\to s}-1)}(\beta)> s \qquad \text{and} \qquad  G_{\sigma}^{(L_{\beta\to s})}(\beta)\leq s.
\end{align}
Since $G$ is non-decreasing on $[0,1]$, we obtain 
\begin{align}
L_{s\to\alpha}-1\leq L_{G_{\sigma}^{(L_{\beta\to s})}(\beta)\to\alpha} \leq  L_{s\to\alpha}.
\end{align}
Putting things together, we get
\begin{align}
L_{\beta\to s}+L_{s\to\alpha}-1\leq L_{G_{\sigma}^{(L_{\beta\to s})}(\beta)\to\alpha} \leq  L_{\beta\to s}+L_{s\to\alpha}.
\end{align}
Applying Lemma \ref{lem:path-upp-bd} and Lemma \ref{lem:path-low-bd} and to $L_{\beta\to s}$ and $L_{s\to \alpha}$, we reach the conclusion.

\end{proof}

\subsubsection{Memorization Conditions}

\begin{proposition}[$\kappa$-memorization and $\epsilon$-closeness] If $\mathbf{K}$ has the $\kappa$-memorization property, then $\mathbf{K}$ satisfies the $\kappa$-closeness property. Conversely, if $\mathbf{K}$ satisfies the $\kappa/n$-closeness property, then $\mathbf{K}$ has the $\kappa$-memorization property.
\end{proposition}

\begin{proof} \textbf{Sufficient condition}. Let $(\lambda,\mathbf{v})$ be a pair of eigenvalue and eigenvector of $\mathbf{K}$. Let $i=\arg\max_{j} |v_{j}|$. Then, we have 
\begin{align}
\lambda_{i}v_{i}=v_{i}+\sum_{j\neq i}v_{j}K_{ij}.
\end{align}
Then, we have
\begin{align}
|1-\lambda|\cdot |v_{i}| \leq \sum_{j\neq i}|v_{j}|\cdot |K_{ij}| \leq (n-1)\cdot\frac{\kappa}{n} |v_{i}| \cdot <\kappa |v_{i}|,
\end{align}
and, thus, $|1-\lambda|\leq \kappa$.

\textbf{Necessary condition}. We have 
\begin{align}
  (1-\kappa)I_{n} \preceq \mathbf{K}\preceq(1+\kappa)I_{n}.
\end{align}
Let $u := (e_{i} - e_{j})/\sqrt{2}$ for $i\neq j$. Then, we have
\begin{align}
  u^\top \mathbf{K}u=1-K_{ij}\in (1-\kappa, 1+\kappa),
\end{align}
and, thus, $|K_{ij}|\leq \kappa$ for all $i\neq j$.
\end{proof}

\begin{proposition}[$\epsilon$-closeness and depth] \label{prop:close_depth} The empirical kernel matrix $\mathbf{K}$ satisfies the $\epsilon$-closeness property if and only if $L\geq L_{\rho\to \epsilon}$. Moreover, we have $\tilde{L}_{\epsilon}=L_{\rho\to \epsilon}$.
\end{proposition}
\begin{proof} $K$ satisfying the $\epsilon$-closeness property means
\begin{align}
  G_{\sigma}^{(L)}(\rho)=\max_{ij}|G_{\sigma}^{(L)}(\rho_{ij})|\leq \epsilon,
\end{align}
which is equivalent by definition to $L\geq L_{\rho\to\epsilon}$.
\end{proof}

\begin{proposition} \label{prop:mem_clos}Recall that is $L_{\kappa}$ is the minimum compositional depth such that the empirical kernel matrix $\mathbf{K}$ has the $\kappa$-memorization property. Then, we have 
\begin{align}
L_{\rho\to \kappa}\leq L_{\kappa}\leq L_{\rho\to\frac{\kappa}{n}}.
\end{align}
\end{proposition}
\begin{proof} By Proposition \ref{prop:mem_clos}, we have that $\kappa$-memorization implies $\kappa$-closeness. Therefore, by Proposition \ref{prop:close_depth}, we have $L_{\kappa}\geq L_{\rho\to \kappa}$. Again, by Proposition \ref{prop:mem_clos}, we have $\frac{\kappa}{n}$-closeness implies $\kappa$ memorization.  Thus, by Proposition \ref{prop:close_depth} and minimality of $L_{\kappa}$, we have $L_{\kappa}\leq L_{\rho\to \frac{\kappa}{n}}$.
\end{proof}

\subsubsection{Small correlation}

Consider the small correlation regime \ref{data:r1} when $\frac{\log n}{d(n)} < c$ with some small enough constant $c<1$. We consider a typical instance of the data where such small correlation is guaranteed with high probability, due to the following lemma.

\begin{lemma}[Concentration: small correlation regime]
  \label{lem:concentration}
  Suppose that $x_i \stackrel{iid}{\sim} {\rm Unif}(\mathbb{S}^{d-1})$ with $i\in [n]$ and  that $ d^{-1} \log n <0.01$,  then
  \begin{align}
    0.5  \sqrt{\frac{\log n}{d}}  \leq \max_{i \neq j}  |\langle x_i, x_j \rangle| \leq 3 \sqrt{\frac{\log n}{d}}
  \end{align}
  with probability at least $1 - 4 n^{-1/2}$. 
\end{lemma}
\begin{proof}[Proof of Lemma~\ref{lem:concentration}]
   Let $g_i \stackrel{iid}{\sim} \cN(0, \mathbf{I}_d)$, then $x_i \sim g_i/\| g_i\|$. It is clear that $\langle x_i, x_j \rangle = \langle x_i, g_j \rangle / \| g_j\|$. Let's bound $|\langle x_i, g_j \rangle|$ and $ \| g_j\|$ separately. Since $\langle x_i, g_j \rangle|x_{i}\sim  \cN(0,1)$, by upper and lower bounds on Gaussian tails, we have for all $\delta\geq0$ that
  \begin{align}
    \sqrt{\frac{2}{\pi}} \frac{\delta}{\delta^2+1} \exp\{-\delta^2/2\}\leq \PP\big(|\langle x_i, g_j \rangle| \geq \delta ~|~ x_i\big) \leq  \sqrt{\frac{2}{\pi}} \frac{1}{\delta} \exp\{-\delta^2/2\} .
  \end{align}
  Marginalizing over $x_{i}$, we get 
  \begin{align}
    \sqrt{\frac{2}{\pi}}\frac{\delta}{\delta^2+1} \exp\{-\delta^2/2\}\leq \PP\big(|\langle x_i, g_j \rangle| \geq \delta \big) \leq   \sqrt{\frac{2}{\pi}} \frac{1}{\delta}  \exp\{-\delta^2/2\}.
  \end{align}
  Since $\|g_{j}\|^{2}\sim\chi_{d}^{2}$, by Chi-square tail bounds (Laurent-Massart), we have for all $\gamma > 0$ that
  \begin{align}
  \PP(\|g_{j}\|^2 \geq d + 2\sqrt{d\gamma} + 2\gamma )\leq \exp\{-\gamma\} \qquad\text{ and }\qquad 
  \PP(\|g_{j}\|^2 \leq d- 2\sqrt{d\gamma} )\leq \exp\{-\gamma\}.
  \end{align}
  \textbf{Upper bound.} For $\delta=\sqrt{5\log n}$, we have by union bound
  \begin{align}
    \PP\big(\max_{i\neq j}|\langle x_i, g_j \rangle|  >\delta\big) 
    &\leq \sum_{i\neq j}\PP\big(|\langle x_i, g_j \rangle| > \delta \big) \\
    &\leq n(n-1)\exp(-\delta^2/2)\\
    &< n^{-1/2}.
  \end{align}
  Therefore, w.p. at least $1-n^{-1/2}$, we have
  \begin{align}
    \max_{i\neq j}|\langle x_i, g_j \rangle| \leq \sqrt{5}\cdot \sqrt{\log n} .
  \end{align}
  For $\gamma=1.5\log n$, we have
  \begin{align}
    \PP(\min_{j}\| g_j \| \leq \sqrt{d} \sqrt{1 - 2\sqrt{\frac{\gamma}{d}}} \big) 
    &\leq \sum_{j} \PP(\| g_j \|^2 \leq  d - 2\sqrt{d \gamma } \big)\\
    &\leq n\exp\{-\gamma \}\\
    &=n^{-1/2} .
  \end{align}
    Therefore, w.p. at least $1-n^{-1/2}$, we have
  \begin{align}
    \min_{j}|g_{j}| \geq \sqrt{d}\sqrt{1- 2\sqrt{1.5} \sqrt{\frac{\log n}{d}}}> 0.86\sqrt{d}.
  \end{align}
  Hence with probability $1 - 2n^{-1/2}$, we know
  \begin{align}
    \min_{j}|g_{j}| \geq 0.86\sqrt{d} ~~\text{and}~~  \max_{i\neq j}|\langle x_i, g_j \rangle| \leq \sqrt{5} \sqrt{\log n}  .
  \end{align}
  Putting together, we complete the proof for the upper bound since
  \begin{align}
    \max_{i \neq j} |\langle x_1, x_j \rangle| \leq  \frac{\max_{i \neq j} |\langle x_i, g_j \rangle| }{ \min_{j} \| g_j \| } \leq 3\sqrt{\frac{\log n}{d}}.
  \end{align}
  \textbf{Lower bound}. For $\delta = \sqrt{\log n}$, we get
  \begin{align}
    \PP \big( \max_{j\neq 1} |\langle x_1, g_j \rangle| < \delta ~|~ x_1 \big)
    &= \prod_{j=2}^{n} \left[ 1 -\PP\big(|\langle x_1, g_j \rangle| \geq \delta ~|~ x_1 \big)  \right] \\
    &\leq \left[ 1-\sqrt{\frac{2}{\pi}}\frac{\delta}{\delta^2+1} \exp\{-\delta^2/2\} \right]^{n-1}\\ 
    &\leq \exp\left\{-(n-1)\sqrt{\frac{2}{\pi}}\frac{\delta}{\delta^2+1} \exp\{-\delta^2/2\} \right\}\\
    &\leq n^{-1/2}.
  \end{align}
  Marginalizing over $x_1$, we have
  \begin{align}
   \PP \big( \max_{j\neq 1} |\langle x_1, g_j \rangle| < \delta \big)\leq n^{-1/2}.
  \end{align}
  Taking $\gamma=1.5\log n$, we have
  \begin{align}
    \PP(\max_{j}\| g_j \| \geq \sqrt{d} (1+ \sqrt{\frac{2 \gamma}{d}}) \big) 
    &\leq \sum_{j}\PP(\| g_j \|^2 \geq d + 2\sqrt{d \gamma} + 2\gamma  \big)\\
    &\leq n\exp\{-\gamma\}\\
    &=n^{-1/2}.
  \end{align}
    Therefore, w.p. at least $1-n^{-1/2}$, we have
  \begin{align}
    \max_{j}|g_{j}| \leq \sqrt{d}\Big(1+\sqrt{3}\sqrt{\frac{\log n}{d}}\Big)\leq2\sqrt{d} .
  \end{align}
  Hence with probability $1 - 2n^{-1/2}$, we know
  \begin{align}
    \max_{j}|g_{j}| \leq 2\sqrt{d} ~~\text{and}~~ \max_{j\neq 1} |\langle x_1, g_j \rangle| \geq \sqrt{\log n} .
  \end{align}
  Putting things together, we complete the proof for the lower bound since
  \begin{align}
    \max_{i \neq j} |\langle x_1, x_j \rangle|  \geq \max_{j \neq 1}  |\langle x_1, x_j \rangle|  \geq  \frac{\max_{j \neq 1} |\langle x_1, g_j \rangle| }{ \max_{j \neq 1} \| g_j \| }\geq 0.5\sqrt{\frac{\log n}{d}} .
  \end{align}
\end{proof}

\begin{lemma}[Restating Lemma~\ref{thm:clos_small}]
  Consider a dataset with random instance $\cX = \{ x_i \stackrel{\text{i.i.d.}}{\sim} {\rm Unif}(\mathbb{S}^{d-1}) \colon i \in [n] \}$.
   Consider the regime $\frac{\log n}{d(n)} < c$ with some absolute constant $c<1$ small enough that only depends on the activation $\sigma(\cdot)$. For any $0<\epsilon<\rho$, we have with probability at least $1-4n^{-1/2}$ that 
  \begin{align}
  \frac{\frac{1}{2}\log\frac{\log{n}}{d} +\log(0.5\epsilon^{-1})}{\log a_{1}^{-2}}\leq \tilde{L}_{\epsilon}\leq 2\cdot \frac{\frac{1}{2} \log\frac{\log{n}}{d} +\log(3\epsilon^{-1})}{\log a_{1}^{-2}} + 1.
  \end{align}
\end{lemma}

\begin{proof}  By Proposition \ref{prop:close_depth}, $\tilde{L}_{\epsilon}=L_{\rho\to\epsilon}$. By Corollary \ref{lem:path-upp-bd} and Corollary \ref{lem:path-low-bd}, we have
  \begin{align}
  \frac{\log \rho+\log \epsilon^{-1}}{\log\frac{\epsilon}{G_{\sigma}(\epsilon)}} \leq H_{\mathrm{low}}(\epsilon,\rho)\leq L_{\rho\to\epsilon}\leq H_{\mathrm{upp}}(\epsilon,\rho)\leq \frac{\log(\rho) +\log(\epsilon^{-1})}{\log\frac{\rho}{G_{\sigma}(\rho)}} + 1.
  \end{align}
  By Lemma \ref{lem:concentration}, we have w.p. at least $1-4n^{-0.5}$ that
  \begin{align}
  \log(0.5)+\frac{1}{2}\log\frac{\log n}{d}\leq \log(\rho)\leq \log(3)+\frac{1}{2}\log\frac{\log{n}}{d},
  \end{align}
  when $\frac{\log n}{d} < 0.01$.
  We also have
  \begin{align}
  \frac{G_{\sigma}(\rho)}{\rho} < a_{1}^2 + (1-a_1^2)\rho 
  \qquad \text{and} \qquad
  \frac{G_{\sigma}(\epsilon)}{\epsilon} > a_{1}^2 
  \end{align}
  Clearly, if $\rho < a_1-a_1^2 $, then we have
  \begin{align}
    \frac{1}{\log \frac{\rho}{G_{\sigma}(\rho)}} < \frac{1}{ \log \frac{1}{a_1^2 + \rho}} < \frac{2}{ \log a_1^{-2}} .
  \end{align}
  To put things together, if $\sqrt{\frac{\log n}{d}} < c\coloneqq\min\{ 0.1, 0.5(a_1-a_1^2)  \}$,
  we have
  \begin{align}
  \frac{\frac{1}{2}\log\frac{\log{n}}{d} +\log(0.5\epsilon^{-1})}{\log a_{1}^{-2}} \leq L_{\rho\to \epsilon}\leq 2\cdot \frac{\frac{1}{2} \log\frac{\log{n}}{d} +\log(3\epsilon^{-1})}{\log a_{1}^{-2}} + 1.
  \end{align}
\end{proof}

\begin{proof}[Proof of Theorem~\ref{thm:small-correlation}]
This is a consequence of Lemma \ref{thm:clos_small} and Proposition \ref{prop:mem_clos}.
\end{proof}

\subsubsection{Large correlation}

We will be in large correlation regime \ref{data:r2}, when $\frac{\log n}{d(n)} > C$ with some large enough constant $C>1$. We consider a typical instance of the data where such large correlation arises in the sphere packing/covering setup.

\begin{definition}[r-covering]
  For a compact subset $V\subset \mathbb{R}^{d}$, we say $\cX=\{x_{i}\}_{i\in[n]}\subset V$ is a $r$-covering of $V$ if for any $x\in V$, there exists $x_{i}\in\cX$ such that $\|x-x_{i}\|\leq r$. We define the covering number of $V$ as $\cN_{r}(V)$, that is
  \begin{align}
  \mathcal{N}_{r}(V)=\min\{ n\colon \text{ exists } r \text{-covering of } V \text{ of size }n\}\enspace.
  \end{align}
\end{definition}

\begin{definition}[r-packing]
  For a subset $V\subset \mathbb{R}^{d}$, we say $\cX=\{x_{i}\}_{i\in[n]}\subset V$ is a $r$-packing of $V$ if for all $x_{i}\neq x_{j}\in\cX$, we have $\|x_{i}-x_{j}\|> r$. We define the packing number of $V$ as $\cK_{r}(V)$, that is
  \begin{align}
  \mathcal{K}_{r}(V)=\max\{ n\colon \text{ exists } r \text{-packing of } V \text{ of size }n\}\enspace.
  \end{align}
\end{definition}

\begin{lemma}[Metric entropy on the sphere]
  \label{lem:sphere-entropy}
  Denote by $S_{d-1} \coloneqq \frac{2 \pi^{d/2}}{\Gamma(d/2)}$ to be the surface area of $\mathbb{S}^{d-1}$. Then, we have for all $r\leq 1$ the following bounds on the packing and covering numbers of a  sphere
  \begin{align}
    \label{eq:metric-entropy-bound}
  \left( \frac{1}{r} \right)^{d-1} \leq \frac{(d-1)S_{d-1}}{S_{d-2}}  \left( \frac{1}{r} \right)^{d-1} \leq \mathcal{N}_r(\mathbb{S}^{d-1}) \leq \mathcal{K}_r(\mathbb{S}^{d-1}) \leq \left( \frac{1+r/2}{r/2} \right)^{d} \leq \left(\frac{3}{r}\right)^d.
  \end{align}
\end{lemma}
\begin{proof}[Proof of Lemma~\ref{lem:sphere-entropy}]
  The upper bound follows since
  \begin{align}
  \mathcal{K}_r(\mathbb{S}^{d-1}) \leq \mathcal{K}_r(\mathbb{B}^{d})\leq \frac{\vol\left(\mathbb{B}^{d}+\frac{r}{2}\mathbb{B}^{d}\right)}{\vol\left(\frac{r}{2}\mathbb{B}^{d}\right)}=\left(\frac{1+r/2}{r/2}\right)^{d} \enspace.
  \end{align}
  Moreover, recall that every minimal $r$-packing is a maximal $r$-covering, thus, $\cN_{r}\leq \cK_{r} $.

  Now, for the lower bound, let's calculate the surface area of the hyper-spherical cap of a covering ball of radius $r$ defined in the following way
  \begin{align}
    {\rm Cap}_r(x_0) := \{ x |  x \in \mathbb{S}^{d-1}, \| x - x_0\| \leq r \}
  \end{align}
  for any $x_0 \in \mathbb{S}^{d-1}$.  By Pythagoras theorem, one can calculate the height of the cap as $h=\frac{r^2}{2}$. The bottom of the cap is a $d-2$ dimensional sphere of radius $f(h) := \sqrt{1 - (1-h)^2}$. Therefore, the surface area of the cap ${\rm Cap}_r(x_0)$ is

  \begin{align}
  S_{d-1,r}
  &:=\int_{0}^{h} S_{d-2}  \left[ f(x) \right]^{d-2} \sqrt{ 1 + \left[ f'(x) \right]^2 } dx \\
  & =  S_{d-2} \int_{0}^{h} x^{\frac{d-3}{2}} (2-x)^{\frac{d-3}{2}} dx \\
  & \leq S_{d-2}  \frac{1}{d-1} (2x)^{\frac{d-1}{2}} \Big|_{0}^{h} =  \frac{S_{d-2}}{d-1} r^{d-1}.
  \end{align}
  Therefore, we have
  \begin{align}
  \cN_{r}(\mathbb{S}^{d-1})\geq \frac{S_{d-1}}{S_{d-1,r}}\geq \frac{(d-1)S_{d-1}}{S_{d-2}} \cdot r^{-(d-1)}.
  \end{align}
  By  Gautschi's inequality, for $d \geq 3$, we have
  \begin{align}
  \frac{(d-1)S_{d-1}}{S_{d-2}} 
  &=(d-1)\sqrt{\pi}\cdot\frac{\Gamma(\frac{d-2}{2}+\frac{1}{2})}{\Gamma(\frac{d-2}{2}+1)}>(d-1)\sqrt{\pi}\cdot\left(\frac{d-2}{2}+1\right)^{\frac{1}{2}-1}\geq 1.
  \end{align}
\end{proof}

\begin{lemma}[Polarized metric entropy on the sphere]
  \label{lem:pol-sphere-entropy}
  For all $r\leq 1$ and $d \geq 3$, we have the following bounds on the packing polarization number of the  sphere
  \begin{align}
  \left(\frac{1}{3r}\right)^{d-1}\leq\frac{1}{2}\left(\frac{1}{2r}\right)^{d-1}\leq \frac{1}{2}\mathcal{K}_{2r}(\mathbb{S}^{d-1})\leq \mathcal{P}_{r}(\mathbb{S}^{d-1})\leq \mathcal{K}_{r}(\mathbb{S}^{d-1})\leq\left(\frac{3}{r}\right)^{d}\enspace.
  \label{eq:pol-metric-entropy-bound}
  \end{align}
\end{lemma}
\begin{proof} By definition, any $r$-packing polarization set is a $r$-packing set. Therefore, the upper bounds $\mathcal{P}_{r}\leq \mathcal{K}_{r}$ follows. Now, consider any $2r$-maximal packing set $\cX=\{x_{1},\ldots,x_{n}\}$ of the sphere. Define the polar neighborhood of $x_{i}$ of radius $r$ as
\begin{align}
{\rm PS}(x_i) \coloneqq \{ x \in \mathbb{S}^{d-1} \colon \|x - x_i\| \leq r\text{ or } \|x + x_i\| \leq r\}
\end{align}
Consider a simple graph $\mathcal{G}$ with vertices $\mathcal{X}$ and edges $(x_{i},x_{j})$ if $x_j \in {\rm PS}(x_i)$ or $x_i \in {\rm PS}(x_j)$. We claim that each vertex has at most one edge. Otherwise, consider a vertex $x_{i}$ adjacent to both $x_{k}$ and $x_{\ell}$.  Then, we must have
\begin{align}
\|x_{\ell}+x_{i}\|\leq r \quad\text{and}\quad \|x_{k}+x_{i}\|\leq r\enspace,
\end{align}
 Thus, $\norm{x_{\ell}-x_{k}}\leq 2r$, which contradicts the definition of a packing set. Therefore, we would need to remove at most $\lfloor\frac{n}{2}\rfloor$ vertices of the graph $\mathcal{G}$ to get a graph $\mathcal{G'}\subset\mathcal{G}$ with no edges. Therefore, the graph $\mathcal{G}'$ has at least $\frac{n}{2}$ vertices and the vertices form an $r$-polarized packing set. The rest of the conclusions follow by Lemma \ref{lem:sphere-entropy}.
\end{proof}

\begin{corollary} [Polarized metric entropy: large correlation regime]
  \label{lemm:upp-low-bd}
  Suppose that the $n$ data points $\cX=\{x_{i}\}_{i\in[n]}$ form a maximum $r$-polarized packing set of $\mathbb{S}^{d-1}$ with $r = r(n,d)$ as a function of $(n, d)$. Then, we have 
  \begin{align}
    1 - 18.2\cdot\exp\left\{- 2 \frac{\log n}{d}\right\}<\max_{i\neq j} |\langle x_i, x_j \rangle| <  1 - 0.06\cdot \exp\left\{-2 \frac{\log n}{d-1}\right\} 
    \label{eq:upp-low-bd}
  \end{align}
\end{corollary}

\begin{proof} By theorem Lemma \ref{lem:pol-sphere-entropy} Equation~\eqref{eq:pol-metric-entropy-bound}, we have
  \begin{align}
   \frac{1}{3}\exp\left\{-\frac{\log n}{d-1}\right\}\leq r(n,d)\leq  3\exp\left\{-\frac{\log n}{d}\right\}\enspace.
  \end{align}
  Let $\cX=\{x_{i}\}_{i\in[n]}$ a maximum $r$-polarized packing set of $\mathbb{S}^{d-1}$. Denote the polarization completion of the set $\cX$ by $\tilde{\cX}$, that is
  \begin{align}
  \tilde{\cX}=\{\tilde{x}_{i}\}_{i\in[2n]}\coloneqq\{x_{i}\}_{i\in[n]}\cup \{-x_{i}\}_{i\in[n]}\enspace.
  \end{align}
  Observe the following simple relationship that allows us to work with the set $\tilde{\cX}$ in this proof 
  \begin{align}
  \max_{i\neq j} |\langle x_{i},x_{j}\rangle|=\max_{i\neq j} \langle\tilde{x}_{i},\tilde{x}_{j}\rangle\enspace.
  \end{align}
  By definition,  we know $\min_{i \neq j} \| \tilde{x}_i- \tilde{x}_j \| > r(n,d)$. Furthermore, we claim that
  \begin{align}
    \min_{i \neq j} \| \tilde{x}_i - \tilde{x}_j \| \leq 2.01 \cdot r(n, d) \enspace. \label{claim:packing:up-bound} 
  \end{align}
  Therefore, the conclusion's upper bound follows because
  \begin{align}
   \langle \tilde{x}_i, \tilde{x}_j \rangle = \frac{2 - \| \tilde{x}_i - \tilde{x}_j \|^2 }{2} < 1 - \frac{\left[r(n,d)\right]^2}{2}\enspace.
  \end{align}
  The claim \eqref{claim:packing:up-bound} can be proved by contradiction. Suppose the claim is not true, then we have  $\min_{i\neq j}\| \tilde{x}_i - \tilde{x}_j \| > 2.01 \cdot r(n, d)$. In an annulus around some $\tilde{x}_{i}$, 
  \begin{align}
  \{ x \in \mathbb{S}^{d-1}\colon r(n,d)< \| x - \tilde{x}_i \| \leq 1.005\cdot r(n, d) \}\enspace,
  \end{align} one can add another point $x'$, such that 
  \begin{align}
  \min_{j\in [2n]\setminus\{i\}} \| x' - \tilde{x}_j \| \geq  \min_{j\in [2n]\setminus\{i\}}\|\tilde{x}_{j}-\tilde{x}_{i}\|- \|\tilde{x}_{i}-x'\|>1.005 \cdot r(n, d) > r(n,d)\enspace.
  \end{align} 
  The above contradicts with the maximal cardinality nature of the polarized $r(n,d)$-packing set. Thus, we know
  \begin{align}
    \max_{i \neq j} \langle \tilde{x}_i, \tilde{x}_j \rangle \geq 1 - \frac{2.01^2}{2} \left[r(n,d)\right]^2 > 1 - 18.2 \cdot \exp\left\{ -2 \frac{\log n}{d} \right\} \enspace.
  \end{align}
\end{proof}

\begin{lemma}[Restating Lemma~\ref{thm:clos_large}]
  Consider a size-$n$ dataset $\cX = \{ x_i \in \mathbb{S}^{d-1}\colon i\in [n]\}$ that forms a certain maximal polarized packing set of the sphere $\mathbb{S}^{d-1}$. And, consider the regime $\frac{\log n}{d(n)} > C$ with some absolute constant $C>1$ large enough. For any $0<\epsilon<\rho$, we have that
  \begin{align}
  &\tilde{L}_{\epsilon}\geq  \max_{s\in(\epsilon,\rho)} \left\{ \frac{\log(\epsilon^{-1})+\log(s)}{\log a_{1}^{-2}}+\frac{2 \frac{\log n}{d}+\log \left(\frac{1-s}{18.2} \right) }{\log \mu} \right\} - 1,\\
  &\tilde{L}_{\epsilon}\leq  \min_{s\in (\epsilon,\rho)} \left\{ \frac{\log(\epsilon^{-1})+\log(s)}{\log\frac{s}{G_{\sigma}(s)}}+ \frac{ 2\frac{\log n}{d-1}+\log\left(\frac{1-s}{0.06} \right) }{\log \frac{1-G_{\sigma}(s)}{1-s}} \right\} + 2.
  \end{align}
\end{lemma}

\begin{proof} 
By Proposition \ref{prop:close_depth}, we have $\tilde{L}_{\epsilon}=L_{\rho\to\epsilon}$. By Corollary \ref{cor:path}, we have
\begin{align}
\max_{s\in[\epsilon,\rho]} H_{\mathrm{low}}(\epsilon,s)+H_{\mathrm{low}}(s,\rho)-1\leq L_{\rho\to \epsilon} \leq \min_{s\in[\epsilon,\rho]} H_{\mathrm{upp}}(\epsilon,s)+  H_{\mathrm{upp}}(s,\rho) 
\end{align}
By Theorem \ref{lemm:upp-low-bd}, we have
\begin{align}
-\log(18.2)+\frac{2\log n}{d}\leq -\log(1-\rho)\leq -\log(0.06)+\frac{2\log n}{d-1}.
\end{align}
For the upper bound, we have
\begin{align}
 &H_{\mathrm{upp}}(\epsilon,s)\leq \frac{\log(\epsilon^{-1})+\log(s)}{\log \frac{s}{G_{\sigma}(s)}},\\
 &H_{\mathrm{upp}}(s,\rho)\leq \frac{-\log(1-\rho)+\log(1-s)}{\log \frac{1-G_{\sigma}(s)}{1-s}}.
\end{align}
For the lower bound, we have
\begin{align}
 & H_{\mathrm{low}}(\epsilon,s)\geq \frac{\log(\epsilon^{-1})+\log(s)}{\log \frac{\epsilon}{G_{\sigma}(\epsilon)}},\\
 & H_{\mathrm{low}}(s,\rho) \geq \frac{-\log(1-\rho)+\log(1-s)}{\log \frac{1-G_{\sigma}(\rho)}{1-\rho}}.
\end{align}
In addition, for the lower bound we have
\begin{align}
\frac{\epsilon}{G_{\sigma}(\epsilon)}\leq a_{1}^{-2}, \qquad \text{ and } \qquad \frac{1-G_{\sigma}(\rho)}{1-\rho}\leq \mu.
\end{align}
Putting the above things together, we get the conclusion.
\end{proof}

\begin{proof}[Proof of Theorem~\ref{thm:large-correlation}] This is a consequence of Lemma \ref{thm:clos_large} and Proposition \ref{prop:mem_clos}. Take $s_\star := \inf\left\{s\geq 0 ~:~ \frac{1 - G_\sigma(s)}{1-s} \geq \frac{1+\mu}{2}  \right\}$ is a constant in $(0, 1)$ that only depends on $\sigma(\cdot)$.  Since $g(s)\coloneqq\frac{1 - G_\sigma(s)}{1-s}$ is continuous and increasing with 
$\lim_{s\to 0} g(s)=1$ and $\lim_{s\to 1} g(s)=\mu$, we have that 
  \begin{align}
  G_{\sigma}(s_{\star}) = 1-(1-s_{\star})\frac{\mu+1}{2}.
  \end{align}
  
  For choices of $s=\frac{1}{2}$ for lower bound and $s=s_{\star}$ for upper bound, we have
    \begin{align}
    &L_{\kappa}\geq  \frac{2 \frac{\log n}{d} - \log 40}{\log \mu} + \frac{\log(0.5 \kappa^{-1}) }{\log a_{1}^{-2}} - 1,\\
    &L_{\kappa}\leq \frac{2 \frac{\log n}{d-1}+\log \frac{1 -s_\star}{0.06}}{\log \frac{1+\mu}{2}} + \frac{\log(s_\star n\kappa^{-1})}{\log\frac{s}{G_{\sigma}(s_\star)}} + 2.
    \end{align}
  Recall that $\frac{\log n}{d(n)} > C$. With the choice $C = \max\{ 1.5 \log 40, \log \frac{1-s_\star}{0.06} \}$, we know
  \begin{align}
    2 \frac{\log n}{d} - \log 40 \geq 1.5 \frac{\log n}{d}, ~~\text{and}~~ 2 \frac{\log n}{d-1}+\log \frac{1 -s_\star}{0.06} \leq 3 \frac{\log n}{d-1} .
  \end{align}
\end{proof}

\subsection{Proofs in Section~\ref{sec:spherical-harmonics}}

\begin{proof}[Proof of Theorem~\ref{thm:spectral-decomposition}]
  First, we know for all $t \in [-1, 1]$
  \begin{align}
    | K_{\sigma}^{(L)}(t) | \leq 1  \enspace ,
  \end{align}
  which implies
  \begin{align}
    \int_{-1}^1 |K_{\sigma}^{(L)}(t)| (1-t^2)^{\frac{d-3}{2}} dt < \infty\enspace.
  \end{align}
  Therefore, one can apply the Funk--Hecke formula (Theorem 2.22 in \cite{AH12}) to obtain
  \begin{align}
    \int K_{\sigma}^{(L)} \big( \langle x, z\rangle \big) Y_{k,j}(z) d \tau_{d-1} (z) = \lambda_k Y_{k,j}(x)
  \end{align}
  with 
  \begin{align}
    \lambda_k = \frac{S_{d-2}}{S_{d-1}} \int_{-1}^{1} K_{\sigma}^{(L)}(t) P_{k, d}(t) (1-t^2)^{\frac{d-3}{2}} dt \enspace.
  \end{align}
  Recall that $ K_{\sigma}^{(L)}(t) = \sum_{\ell} \PP(Z_{\sigma, L} = \ell) t^\ell$. Let us calculate the following expression explicitly through the Rodrigues representation formula (Theorem 2.23 and Proposition 2.26 in \cite{AH12}) for the Legendre polynomial, for $\ell \geq k$
  \begin{align}
    \int_{-1}^{1} t^{\ell} P_{k, d}(t) (1-t^2)^{\frac{d-3}{2}} dt &= \frac{\Gamma(\frac{d-1}{2})}{2^k \Gamma(k+\frac{d-1}{2})}  \cdot \int_{-1}^1 \left[ \left(\frac{d}{dt}\right)^k  t^{\ell}\right]  (1-t^2)^{k+\frac{d-3}{2}} dt   \\
    &= \frac{\Gamma(\frac{d-1}{2})}{2^k \Gamma(k+\frac{d-1}{2})} \frac{\ell!}{(\ell-k)!} \int_{-1}^1 t^{\ell-k} (1-t^2)^{k+\frac{d-3}{2}} dt  \\
    &=\frac{\Gamma(\frac{d-1}{2})}{2^k \Gamma(k+\frac{d-1}{2})} \frac{\ell!}{(\ell-k)!} [1+(-1)^{\ell-k}] \int_{0}^1 s^{\frac{\ell-k}{2}}(1-s)^{k+\frac{d-3}{2}} \frac{ds}{2 s^{\frac{1}{2}}} \\
    &=\frac{\Gamma(\frac{d-1}{2})}{\Gamma(k+\frac{d-1}{2})} \frac{\ell!}{(\ell-k)!} \frac{1+(-1)^{\ell-k}}{2^{k+1}} \frac{\Gamma(\frac{\ell-k+1}{2})\Gamma(k+\frac{d-1}{2}) }{\Gamma(\frac{\ell+k+d}{2})} \enspace.
  \end{align}
  Plug in $K_{\sigma}^{(L)}(t) = \sum_{\ell} \PP(Z_{\sigma, L} = \ell) t^\ell$, we obtain the formula for $\lambda_k$.
  Due to the orthogonality of Legendre polynomial under measure $(1-t^2)^{\frac{d-3}{2}}$ \eqref{eq:legendre-ortho}, we know
  \begin{align}
     K_{\sigma}^{(L)} (t) = \sum_{k \geq 0} \lambda_k N_{k,d} P_{k, d}(t)
  \end{align}
  and by the Addition Theorem 2.9 in \cite{AH12}, we know
  \begin{align}
    N_{k,d} P_{k, d}\big( \langle x, z\rangle \big) = \sum_{j=1}^{N_{k,d}} Y_{k,j}(x) Y_{k, j}(z) \enspace.
  \end{align}
  Put things together, we have
  \begin{align}
    K_{\sigma}^{(L)} \big( \langle x, z\rangle \big) = \sum_{k\geq 0} \lambda_k \sum_{j=1}^{N_{k,d}} Y_{k,j}(x) Y_{k,j}(z) \enspace.
  \end{align}
\end{proof}

\subsection{Proofs in Section~\ref{sec:new-random-feature-alg}}
\begin{proof}[Proof of Theorem~\ref{thm:kernels-to-activations}]
  We denote $f \in L^1_{(d-3)/2}: [-1, 1] \to \mathbb{R}$ as $L^1$ integrable in the following sense
  \begin{align}
    \int_{-1}^{1} |f(t)| (1-t^2)^{\frac{d-3}{2}} < \infty \enspace.
  \end{align}
  Note that for $d\geq 2$, continuous function is $L^1$ integrable since $C([-1,1]) \subset L^1_{(d-3)/2}$.
  Due to Equation (2.66) (variant of Funk-Hecke Formula) in \cite{AH12}, we know that for any $g \in L^1_{(d-3)/2}$, $x, z \in \mathbb{S}^{d-1}$ and $k \geq 0$
  \begin{align}
    \int_{\mathbb{S}^{d-1}} g(\xi^\top x) P_{k,d}(\xi^\top z) d\tau_{d-1}(\xi) =  \frac{\beta_k}{N_{k,d}} \cdot P_{k,d}(x^\top z)
  \end{align}
  with $\beta_k$ being the coefficients of $g$ under Legendre polynomials
  \begin{align}
    g(t) = \sum_{k\geq 0} \beta_k P_{k,d}(t) \enspace.
  \end{align}
  Let's first plug in $g = P_{k,d}$, then for any $\alpha_k\geq 0$
  \begin{align}
    & \int_{\mathbb{S}^{d-1}}P_{k,d}(\xi^\top x) P_{k,d}(\xi^\top z) d\tau_{d-1}(\xi) =  \frac{1}{N_{k,d}} \cdot P_{k,d}(x^\top z) \\
    &\int_{\mathbb{S}^{d-1}} \left( \sqrt{\alpha_k N_{k,d}} P_{k,d}(\xi^\top x) \right) \left( \sqrt{\alpha_k N_{k,d}} P_{k,d}(\xi^\top z) \right) d\tau_{d-1}(\xi) =  \alpha_k \cdot P_{k,d}(x^\top z) \enspace. \label{eq:act-k}
  \end{align}
  Then plug in $g = P_{\ell, d}, \ell \neq k$, we know
  \begin{align}
    \int_{\mathbb{S}^{d-1}}P_{\ell,d}(\xi^\top x) P_{k,d}(\xi^\top z) d\tau_{d-1}(\xi) = 0 \enspace. \label{eq:act-k-l}
  \end{align}
  Therefore we know
  \begin{align}
    &\EE_{\mathbf{\xi} \sim \tau_{d-1}} \left[ \sigma_f( \mathbf{\xi}^\top x) \sigma_f(\mathbf{\xi}^\top z) \right] = \sum_{k, \ell \geq 0}  \EE_{\mathbf{\xi} \sim \tau_{d-1}} \left[ \left( \sqrt{\alpha_k N_{k,d}} P_{k, d}( \mathbf{\xi}^\top x) \right) \left( \sqrt{\alpha_\ell N_{\ell,d}} P_{\ell, d}(\mathbf{\xi}^\top z) \right) \right]\\
    & = \sum_{k=0}^{\infty} \int_{\mathbb{S}^{d-1}} \left( \sqrt{\alpha_k N_{k,d}} P_{k,d}(\xi^\top x) \right) \left( \sqrt{\alpha_k N_{k,d}} P_{k,d}(\xi^\top z) \right) d\tau_{d-1}(\xi) \quad \text{Equation~\eqref{eq:act-k-l}} \\
    & = \sum_{k=0}^\infty \alpha_k P_{k,d}(\langle x, z\rangle) = f(\langle x, z\rangle) \quad \text{Equation~\eqref{eq:act-k}} \enspace.
  \end{align}
\end{proof}

\begin{proof}[Proof of Theorem~\ref{thm:ekm-truncation-regularization}]
  The proof relies on the Mehler's formula in Proposition~\ref{prop:mehler}. It is straightforward to verify that the diagonal element of $\mathbf{R}$ is zero. Let's focus on the off-diagonal components $\mathbf{R}[i,\ell]$ with $i, \ell \in [n]$, $i\neq \ell$. Define the following notation
  \begin{align}
    \sigma_{> \iota}(t) :=  \sum_{k>\iota} \sqrt{\alpha_k} h_k(t) \enspace,
  \end{align}
  and observe that
  \begin{align}
    \mathbf{R}[i,\ell] &= \K[i,\ell] - \EE_{\t \sim \cN(0, \I_d)}\big[ \sigma_{\leq \iota}( x_i^\top \t ) \sigma_{\leq \iota}(x_\ell^\top \t ) \big] \\
    & = 2\EE_{\t \sim \cN(0, \I_d)}\big[ \sigma_{> \iota}( x_i^\top \t ) \sigma_{\leq \iota}(x_\ell^\top \t ) \big] + \EE_{\t \sim \cN(0, \I_d)}\big[ \sigma_{> \iota}( x_i^\top \t ) \sigma_{> \iota}(x_\ell^\top \t ) \big] \enspace. \label{eqn:R_i_l}
  \end{align}
  By Mehler's formula and the derivations in Lemma~\ref{lem:key-duality}, we know that for $f, g \in L^2_\phi$ and $\rho_{il} := \langle x_i, x_\ell \rangle$
  \begin{align*}
    &\EE_{\t \sim \cN(0, \I_d)}\big[ f( x_i^\top \t ) g( x_\ell^\top \t ) \big] = \int_{\mathbb{R}}\int_{\mathbb{R}}f(\tilde{x})g(\tilde{z})\frac{1}{2\pi\sqrt{1-\rho_{i\ell}^2}}\exp\left\{-\frac{\tilde{x}^2+\tilde{z}^2-2\rho_{i\ell}\tilde{x}\tilde{z}}{2(1-\rho_{i\ell}^2)}\right\}d\tilde{x}d\tilde{z} \\
    &\quad =  \int_{\mathbb{R}}\int_{\mathbb{R}}f(\tilde{x})g(\tilde{z})\frac{1}{2\pi}\exp\left\{-\frac{\tilde{x}^2+\tilde{z}^2}{2}\right\}\sum_{k\geq0}\rho_{i\ell}^{k}h_{k}(\tilde{x})h_{k}(\tilde{z})\\
    &\quad = \sum_{k\geq 0}\rho_{i\ell}^{k}\EE_{\tilde{x}\sim \cN(0,1)}[f(\tilde{x})h_{k}(\tilde{x})]\EE_{\tilde{z}\sim \cN(0,1)}[g(\tilde{z})h_{k}(\tilde{z})] \enspace.
  \end{align*}
  Plug in $f(\cdot) = \sigma_{>\iota}(\cdot)$ and $g(\cdot) = \sigma_{\leq \iota}(\cdot)$, we know that the first term in \eqref{eqn:R_i_l} is zero. Plug in with $f(\cdot) = g(\cdot) = \sigma_{>\iota}(\cdot)$, we know that that the second term in \eqref{eqn:R_i_l} satisfies
  \begin{align}
     \EE_{\t \sim \cN(0, \I_d)}\big[ \sigma_{> \iota}( x_i^\top \t ) \sigma_{> \iota}(x_\ell^\top \t ) \big] = \sum_{k > \iota} \rho_{i\ell}^k \alpha_k
  \end{align}
  and hence $|\mathbf{R}[i,\ell]| \leq |\rho_{i\ell}|^{\iota+1}$.  Now use concentration results on $|\rho_{i\ell}|$ derived in Lemma~\ref{lem:concentration}, and recall that for small $\delta>0$,
  $d(n) > n^{\frac{2}{\iota+1} + \delta}, $ we have
  \begin{align}
  \| \mathbf{R} \|_{\rm op} \leq n \cdot \max_{i,\ell\in [n], i\neq \ell} |\rho_{i\ell}|^{\iota+1} \precsim n \cdot \left( \frac{\log n}{d} \right)^{\frac{\iota+1}{2}} \leq  \left( \frac{\log n}{n^{ \delta}} \right)^{\frac{\iota+1}{2}} = o(1) \enspace.
  \end{align}
\end{proof}

\end{document}